\documentclass{article}

 \usepackage[main, final]{neurips_2026}

\usepackage[utf8]{inputenc} 
\usepackage[T1]{fontenc}    
\usepackage{hyperref}       
\usepackage{url}            
\usepackage{booktabs}       
\usepackage{amsfonts}       
\usepackage{nicefrac}       
\usepackage{microtype}      
\usepackage{xcolor}         

\usepackage{amssymb}            
\usepackage{mathtools}          
\usepackage{mathrsfs}           
\mathtoolsset{showonlyrefs}     
\usepackage{graphicx}           
\usepackage{adjustbox}
\usepackage{subcaption}         
\usepackage[space]{grffile}     
\usepackage{url}                
\usepackage{lipsum}             
\usepackage{amsthm}          
\usepackage{thmtools}
\usepackage{thm-restate}

\declaretheorem[name=Definition, numberwithin=section]{definition}
\declaretheorem[name=Assumption, numberwithin=section]{assumption}

\usepackage{drafttools}           
\usepackage{amsmath, amsfonts, amssymb, amsthm}
\usepackage{mathtools} 
\usepackage{bm}        

\DeclareMathOperator*{\argmax}{argmax}

\usepackage{acronym}

\newacro{DM}{decision making}
\newacro{AD}{automated driving}
\newacro{SOC}{discrete-time stochastic optimal control}

\newacro{NLP}{nonlinear program}
\newacro{ODE}{ordinary differential equation}
\newacro{MPC}{model predictive control}
\newacro{MHE}{moving horizon estimation}
\newacro{NMPC}{nonlinear model predictive control}
\newacro{ENMPC}{economic nonlinear model predictive control}
\newacro{MPPI}{model predictive path integral control}
\newacro{SPC}{subspace predictive control}
\newacro{QP}{quadratic program}
\newacro{MIQP}{mixed-integer {\ac{QP}}}
\newacro{MILP}{mixed-integer \ac{LP}}
\newacro{MINLP}{mixed-integer \ac{NLP}}
\newacro{MI}{mixed-integer}
\newacro{MIP}{mixed-integer programming}
\newacro{MPCC}{mathematical program with complementarity constraints}
\acrodefplural{MPCC}{mathematical programs with complementarity constraints}
\newacro{BB}{branch-and-bound}
\newacro{SQP}{sequential quadratic programming}
\newacro{RNN}{recurrent neural network}
\newacro{OCP}{optimal control problem}
\newacro{LQR}{linear quadratic regulator}
\newacro{iLQR}{iterative linear quadratic regulator}
\newacro{SLQ}{sequential linear quadratic programming}
\newacro{DDP}{differential dynamic programming}
\newacro{MCP}{mixed complementarity problem}
\newacro{IP}{interior point}
\newacro{ADMM}{alternating direction method of multipliers}
\newacro{RTI}{real time iteration}
\newacro{POMDP}{partially observable \ac{MDP}}
\newacro{CEM}{cross-entropy method}
\newacro{GGN}{generalized Gauss-Newton}
\newacro{KKT}{Karush-Kuhn-Tucker}
\newacro{IFT}{implicit function theorem}

\newacro{RL}{reinforcement learning}
\newacro{DP}{dynamic programming}
\newacro{LSTM}{long short-term memory}
\newacro{NN}{artificial neural network}
\newacro{PPO}{proximal policy optimization}
\newacro{PG}{policy gradient}
\newacro{DPG}{deterministic \ac{PG}}
\newacro{SPG}{stochastic \ac{PG}}
\newacro{TD}{temporal difference}
\newacro{SAC}{soft actor-critic}
\newacro{PI}{policy iteration}
\newacro{SARSA}{state action reward state action}
\newacro{IL}{imitation learning}
\newacro{MLE}{maximum likelihood estimation}
\newacro{MDP}{Markov decision process}
\newacroplural{MDP}[MDPs]{Markov decision processes}
\newacro{DDPG}{deep deterministic \ac{PG}}
\newacro{DQN}{deep Q-networks}
\newacro{TD3}{twin-delayed actor-critic}
\newacro{VI}{value iteration}
\newacro{TRPO}{trust region policy optimization}
\newacro{BC}{behavior cloning}
\newacro{GAN}{generative adversarial network}
\newacro{GPS}{guided policy search}
\newacro{GP}{Gaussian process}
\newacro{leap-c}{Learning for Predictive Control}
\newacro{FA}{function approximator}
\newacro{ADP}{approximate \ac{DP}}
\newacro{LMPC}{linear \ac{MPC}}

\newacro{FB}{forward-backward}
\newacro{SR}{successor representation}
\newacro{SF}{successor features}
\newacro{SVD}{singular value decomposition}

\newacro{FF}{feed forward network}
\newacro{DNN}{deep neural network}
\newacro{OMD}{optimal model design}

\newacro{UGV}{unmanned ground vehicle}

\newacro{RWE}{real-world experiments}

\title{Spectral Alignment in Forward–Backward Representations via Temporal Abstraction}

\author{%
  Seyed Mahdi B. Azad \quad Jasper Hoffmann \quad Iman Nematollahi \quad Hao Zhu\\
  \And
  Abhinav Valada \quad Joschka B\"{o}decker \\
  \\
  Department of Computer Science, University of Freiburg, Germany  \\
  \texttt{{basiri,hoffmaja,nematoli,zhuh,valada,jboedeck}@cs.uni-freiburg.de} \\
}

\begin{document}

\maketitle

\begin{abstract}
    \Acf{FB} representations provide a powerful framework for learning the \acf{SR} in continuous spaces by enforcing a low-rank factorization. 
    However, a fundamental spectral mismatch often exists between the high-rank transition dynamics of continuous environments and the low-rank bottleneck of the \ac{FB} architecture, making accurate low-rank representation learning difficult.
    In this work, we analyze temporal abstraction as a mechanism to mitigate this mismatch.
    By characterizing the spectral properties of the transition operator, we show that temporal abstraction acts analogously to a low-pass filter that suppresses high-frequency spectral components. 
    This suppression reduces the effective rank of the induced SR while preserving a formal bound on the resulting value function error.
    Empirically, we show that this alignment is a key factor for stable \ac{FB} learning, particularly at high discount factors where bootstrapping becomes error-prone. 
    Our results identify temporal abstraction as a principled mechanism for shaping the spectral structure of the underlying \acs{MDP} and enabling effective long-horizon representations in continuous control.
    
    Project page: \url{https://erfi.github.io/spectral_alignment}
\end{abstract}

\begin{figure}[t]
    \begin{center}
        \includegraphics[scale=0.7]{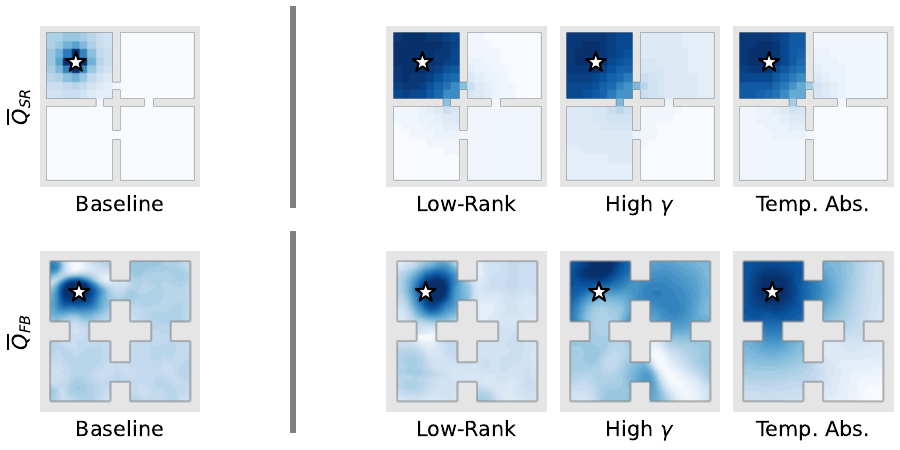}
    \end{center}
    \caption{
    \textbf{Q-function via Successor Representation (SR).} 
    The \ac{SR} enables rapid value inference for arbitrary goals (e.g., \textit{star marker}). Low-rank structure in \ac{SR} is desirable for navigation, as it preserves topological features (e.g., rooms) while suppressing transient dynamics.
    \textbf{Top:} In discrete MDPs, the SR can be computed from the transition matrix.
    \textbf{Bottom:} In continuous domains, \acf{FB} learning approximates the SR, where the embedding dimension controls the rank of the approximation.
    Low-rank structure can arise through (1) explicit constraints (\textit{Low-Rank} column; e.g., SVD or small embeddings), (2) long horizons (\textit{High $\gamma$} column), or (3) temporal abstraction (\textit{Temp.~Abs.} column; e.g., action repetition).
    In continuous settings, temporal abstraction provides the spectral alignment needed for effective bootstrapping, whereas high $\gamma$ or overly restrictive bottlenecks can impair representation learning, leading to Q-functions with many erroneous local maxima. 
    }
    \label{fig:cover}
\end{figure}

\section{Introduction}
\label{sec:introduction}

Effective long-horizon control requires representations that map current actions to future outcomes. The \acf{SR} achieves this by encoding discounted future state--action occupancies \citep{dayan1993improvinggeneralization}, providing a structured foundation for value computation across diverse rewards. While SR-based methods have successfully scaled to high-dimensional control \citep{kulkarni2016DeepSR, zhang2017deeprlwithsuccessor, azad2025srreward}, continuous domains require representations that are both expressive and computationally tractable. \Acf{FB} representations address this by learning a low-rank factorization of the SR directly from interaction \citep{blier2021learning, touati2021learningone}. 

However, a fundamental incompatibility exists: while FB assumes a low-rank constraint, the true SR in continuous environments is often high-rank with slow spectral decay \citep{dubail2025shift}. We identify this spectral mismatch as a primary bottleneck for FB. Empirically, we show that increasing network capacity does not reliably improve performance; instead, higher capacity can lead to performance degradation as networks attempt to resolve high-frequency dynamical components that are inherently difficult to predict. When coupled with bootstrapping, errors in these spectral components propagate through Bellman updates and destabilize the learning process.

To address this, we leverage temporal abstraction via action repetition to regulate the SR's spectral structure. We demonstrate that multi-step transitions accelerate spectral decay, yielding a more structured, low-rank target for the FB objective. As shown in Figure~\ref{fig:cover}, action repetition—previously utilized for exploration and efficiency \citep{mnih2015humanlevel, biedenkapp2021temporl}—consistently improves both representation quality and episodic return across discrete and continuous environments.

Finally, we examine the influence of the discount factor, $\gamma$. While increasing $\gamma$ extends the task horizon, it also degrades SR conditioning and amplifies sub-dominant spectral components, increasing sensitivity to noise. We show that temporal abstraction counteracts this by improving spectral concentration, enabling stable learning at high effective horizons. Together, our results provide a unified perspective on the spectral requirements of FB learning, shifting the burden of representation from the function approximator to the design of interaction dynamics.
\section{Related Works}
\label{sec:related-works}

\Aclp{SR} were introduced as task-agnostic predictive representations that enable rapid adaptation to new reward functions \citep{dayan1993improvinggeneralization}.
Subsequent work has leveraged SR for transfer and zero-shot reinforcement learning \citep{barreto2017SRfortransfer}.
However, exact computation scales poorly with state dimensionality, motivating low-rank and parametric approximations that capture dominant long-horizon dynamics.

\Acl{FB} representation learning methods \citep{blier2021learning, touati2021learningone, touati2023doeszeroshotrl} address this challenge by learning factorizations of the SR that emphasize shared future occupancies over fine-grained state distinctions.
While effective, these approaches implicitly rely on a low-rank SR structure and offer limited theoretical insight into when such a structure arises.
Our work complements FB by linking the effective rank of the SR to the spectral properties of the transition dynamics induced by the policy and environment, and by proposing mechanisms that promote this low-rank structure.
This capacity-dependent view aligns with \citet{brellmann2024ondouble}, who show that increasing model capacity alone does not always improve performance and can destabilize training — a pattern consistent with our own capacity ablations.

The transition operator of a Markov decision process is central to long-term behavior, mixing, and value estimation.
Classical Markov chain theory relates the spectral gap of the transition matrix to convergence rates \citep{meyn2012markov}.
In reinforcement learning, spectral methods have informed representation learning and planning, including proto-value functions and Laplacian-based abstractions \citep{mahadevan2005proto, machado2017Alaplacian, machado2017EigenoptionDT, shehmar2026laplacian}.
However, prior work focuses on policy evaluation and transfer, without analyzing how transition spectra influence the rank, compressibility, or learnability of low-rank SR under function approximation.
A related spectral perspective is taken by \citet{lyle2022learning}, who show that temporal-difference learning preferentially fits high-frequency components, facilitating non-smooth value fitting but reducing generalization. We find a related effect for forward-backward representations, where temporal abstractions suppress high-frequency modes and improve performance. While \citet{lyle2022learning} study the spectrum of value-function learning induced by the reward and dynamics, we analyze the learning of successor representations across multiple rewards and highlight how temporal abstraction improves their spectral properties and final performance.

Temporal abstraction has been widely studied through semi-Markov decision processes and options \citep{sutton1999between}.
A simple instance is action repetition (frame skipping), used in Atari benchmarks \citep{mnih2015humanlevel} and known to significantly affect learning \citep{machado2018revisiting, biedenkapp2021temporl}. 
\citet{li2025reinforcement} goes beyond action repetition and uses action chunking for unbiased n-step returns and smoother execution. In contrast we use action repetition specifically as a spectral regularizer aligning learning targets with FB's low-rank inductive bias.  
From an operator perspective, repeating actions replaces the one-step transition matrix with its $k$-step counterpart, smoothing the dynamics.
Existing work primarily motivates this via efficiency or exploration, without examining its impact on spectral structure or low-rank predictive representations such as the SR.
Action repetition has also been used by \citet{sargent2022temporallyextended} as a mechanism for temporal abstraction with SR-based representations; their motivation is practical and concerns faster adaptation and reduced decision overhead whereas ours is explanatory, characterizing repetition as a mechanism that reshapes the spectral structure of the dynamics.

Although the connection between multi-step transitions and SR spectra is established \citep{dayan1993improvinggeneralization, machado2017EigenoptionDT, machado2017Alaplacian, dubail2025shift}, and FB methods are empirically successful \citep{touati2021learningone, touati2023doeszeroshotrl}, their interaction remains underexplored.
We reinterpret temporal abstraction not merely as an exploration heuristic \citep{lakshminarayanan2017dynamicaction} or a horizon reduction technique~\cite{park2025horizon}, but as a spectral alignment mechanism that bridges high-rank dynamics and the low-rank inductive bias of FB representations.
\section{Background}
\label{sec:background}

We represent a finite, reward-free \ac{MDP} as a tuple $\mathcal{M} = (\mathcal{S}, \mathcal{A}, P, \gamma)$, where $\mathcal{S}$ and $\mathcal{A}$ represent the state and action spaces, respectively, $P(s' \mid s, a)$ is the transition probability from state $s$ to $s'$ given action $a$, and $\gamma \in (0, 1)$ is the discount factor \citep{sutton1998introduction}. 
Given a policy $\pi$, the policy-induced transition operator is defined as a matrix $P^{\pi} \in \mathbb{R}^{|\mathcal{S} \times \mathcal{A}| \times |\mathcal{S} \times \mathcal{A}|}$, where $P^{\pi}(s', a' \mid s, a) = \mathbb{P}(s_{t+1}=s', a_{t+1}=a' \mid s_t=s, a_t=a, \pi)$. The matrix $P^{\pi}$ is row-stochastic, i.e., $P^{\pi} \mathbf{1} = \mathbf{1}$. 
The (discounted) \ac{SR} associated with $P^{\pi}$ is defined as $M^{\pi} = (I - \gamma P^{\pi})^{-1} = \sum_{t=0}^{\infty} \gamma^t (P^{\pi})^t$. 
In the following, we use the matrix $M^\pi$ and its functional form interchangeably. We define $M^\pi(s, a, s', a')$ as the expected discounted occupancy of $(s', a')$ given an initial state-action pair $(s, a)$. In matrix notation, it corresponds to the entry of $M^\pi$ indexed by row $(s, a)$ and column $(s', a')$.

\subsection{Forward-Backward Representation}
The \ac{FB} representation is a parametric framework designed to approximate the \ac{SR} for all optimal policies in an unsupervised way  \citep{touati2021learningone}. Let $(\pi_z)_{z\in \mathbb{R}^d}$ be a family of policies parameterized by $z \in \mathbb{R}^d$, and define the embedding functions $F\colon \mathcal{S} \times \mathcal{A} \times \mathbb{R}^d \to \mathbb{R}^d$ and $B\colon \mathcal{S} \times \mathcal{A} \to \mathbb{R}^d$. Learning an \ac{FB} representation entails finding $(F, B, \pi_z)$ such that:
\begin{equation}\label{eq:fb_criterion}
\pi_z(s) \in \argmax_a F(s, a, z)^\top z \quad \text{and} \quad F(s, a, z)^\top B(s', a') = M^{\pi_z}(s, a, s', a')
\end{equation}
for all $(s, a), (s', a') \in \mathcal{S} \times \mathcal{A}$ and $z\in\mathbb{R}^d$.
In continuous action spaces, the $\argmax$ in Eq.~\eqref{eq:fb_criterion} is intractable. Following standard practice~\citep{touati2021learningone}, we introduce a learned actor $\pi_\theta\colon \mathcal{S}\times\mathbb{R}^d \to \mathcal{A}$ trained jointly with $F$ and $B$ to approximate the maximizer. Architectural details and learning rates are deferred to Appendix~\ref{sec:hyperparameters}.
Further, Eq.~\eqref{eq:fb_criterion} represents a fixed-point condition for the triplet $(F, B, \pi_z)$ since $F$ and $B$ depend on $\pi_z$, and $\pi_z$ is defined via $F$ \citep{touati2021learningone}. 
Given a reward function $r\colon  \mathcal{S} \times \mathcal{A} \to \mathbb{R}$, we define $z_R = B^\top r$.
If the condition holds exactly, the optimal action-value function is recovered by $Q^\star(s, a) = F(s, a, z_R)^\top z_R$.

Given a \ac{FB} representation $(F, B)$, we define the approximate successor representation as $\hat{M}^z(s, a, s', a') = F(s, a, z)^\top B(s', a')$.
The following theorem bounds the approximation error of the optimal action-value function $Q^\star$ by the approximation error in successor representation $M^{\pi_z}$:

\begin{restatable}[Optimality Gap for FB Representations]{theorem}{OptGapThm}
\label{theorem:opt_gap}
Let $r\colon \mathcal{S} \times \mathcal{A} \rightarrow \mathbb{R}$ be a reward function such that $z_R = B^\top r$. The approximation error of the optimal Q-function is bounded by:
\begin{equation}\label{eq:fb_bound}
    \left\| F(\,\cdot\,, \,\cdot\,,\ z_R)^\top z_R - Q^{\star} \right\|_\infty \leq \frac{2 C_\mathrm{norm} \; \|r\|_\infty}{(1 - \gamma)} \| \hat{M}^{z_R} - M^{\pi_{z_R}} \|_{2}.
\end{equation}
\end{restatable}

\noindent Here, $C_\mathrm{norm}$ is a constant arising from the choices of norms, equal to $\sqrt{|\mathcal{S}||\mathcal{A}|}$ in our finite case; see Appendix~\ref{sec:proofs} for details. Further, $\|\cdot\|_{\infty}$ denotes the $L_\infty$ or Chebyshev norm, which for a function $f\colon \mathcal{S} \times \mathcal{A} \to \mathbb{R}$ is defined as $\|f\|_{\infty} = \sup_{(s, a)} |f(s,a)|$. The norm $\|\cdot\|_2$ denotes the $L_2$ or spectral norm of a matrix, defined as $\|M\|_2 = \sup_{x \neq 0} \|Mx\|_2 / \|x\|_2$, which corresponds to the largest singular value of $M$.
Note that Theorem~\ref{theorem:opt_gap} is a simplified version of the result in \cite[Theorem 8]{touati2021learningone} tailored to our spectral analysis setting.


\subsection{Spectral Bound on Approximation Error}
To understand the approximation capacity of the \ac{FB} framework, we derive a lower bound on the approximation error appearing on the right-hand side of Eq.~\eqref{eq:fb_bound} based on the spectrum of $M^{\pi_{z_R}}$.
Related to this is the work in \cite{dubail2025shift}, which performs a similar study with a focus on finite-sample analysis. In this work, we do not aim to derive the tightest possible bound, but rather to develop a simple theoretical framework that highlights the effect of temporal abstractions on the optimal approximation error. We leave a finite-sample analysis to future work.

Due to limited representational capacity when $d$ is small, the \ac{FB} criterion cannot generally be fulfilled exactly, even in the finite case.
Furthermore, the \ac{FB} representation must simultaneously reconstruct the successor representation and define a greedy policy, as shown in Eq.~\eqref{eq:fb_criterion}.
By the Eckart–Young–Mirsky theorem \citep{eckart1936theapproximation}, the best rank-$d$ approximation of $M^{\pi_{z_R}}$ is obtained via the truncated \ac{SVD}, denoted by $M^\star$, which satisfies $\|M^\star - M^{\pi_{z_R}}\|_2 = \sigma_{d+1}(M^{\pi_{z_R}})$. Intuitively, $\sigma_{d+1}(M^{\pi_{z_R}})$ corresponds to the first discarded singular value. Motivated by this observation, we define the following:

\begin{definition}[\Acl{FB} Realization Error]
\label{def:real_error}
Given a reward function $r\colon \mathcal{S} \times \mathcal{A} \to  \mathbb{R}$, and a \ac{FB} representation $(F, B)$, we define the \emph{\ac{FB} realization error} as the difference to the optimal rank $d$ approximation:
$\epsilon_{\mathrm{real}}(r) \coloneqq \| \hat{M}^{z_R} - M^{\pi_{z_R}} \|_2 - \sigma_{d+1}(M^{\pi_{z_R}}) \geq 0$.
\end{definition}
Consequently, the optimality gap in Eq.~\eqref{eq:fb_bound} is governed by the decay of the representation's singular values,
\begin{equation}\label{eq:fb_spectral_bound}
    \left\| F(\,\cdot\,, \,\cdot\,,\ z_R)^\top z_R - Q^{\star} \right\|_\infty \leq \frac{2 C_\mathrm{norm}\; \|r\|_\infty}{(1 - \gamma)} \left(\epsilon_\mathrm{real}(r) + \sigma_{d + 1}(M^{\pi_{z_R}})\right),
\end{equation}
assuming that the error $\epsilon_\mathrm{real}(r)$ is sufficiently small.
This decomposition separates the FB realization error $\epsilon_\mathrm{real}(r)$ from the spectral truncation error $\sigma_{d+1}(M^{\pi_{z_R}})$ determined by the singular values of the successor representation.

\section{Temporal Abstraction in Forward-Backward Representations} 
\label{sec:analysis}

Our goal is to demonstrate that temporal abstraction is beneficial for learning \ac{FB} representations.
To this end, we introduce a simple temporal abstraction, namely action repetition.
Action repetition was introduced in \cite{mnih2015humanlevel} and has been shown to be beneficial for exploration and learning performance in model-free \ac{RL} \citep{biedenkapp2021temporl}.


\subsection{Action Repetition for Temporal Abstraction}

In the following, we first formally introduce the concept of action-repeat \acp{MDP}, provide the necessary assumptions for this work, and conclude by connecting these concepts to the \ac{FB} representation.

\begin{definition}[Action-Repeat MDP]
\label{def:action_repeat_mdp}
Given a reward-free MDP $\mathcal{M} = (\mathcal{S}, \mathcal{A}, P, \gamma)$, an action-repeat MDP $\widetilde{\mathcal{M}}$ with repeat factor $k \in \mathbb{N}$ is defined by the tuple $(\mathcal{S}, \mathcal{A}, \widetilde{P}, \gamma^k)$. 
The transition probability $\widetilde{P}(s'|s, a)$ represents the probability of reaching state $s'$ after executing action $a$ for $k$ consecutive time steps in $\mathcal{M}$. Mathematically, this is the $k$-fold composition of the transition operator:
\begin{equation}
    \widetilde{P}(s'|s, \textcolor{purple}{a}) = \sum_{(s_1, \dotsc, s_{k-1}) \in S^{k-1}} P(s'|s_{k-1}, \textcolor{purple}{a}) \cdots P(s_1|s, \textcolor{purple}{a}).
\end{equation}
Note that for $k=1$ we define $\widetilde{P}(s' | s, a) = P(s'|s,a)$.
\end{definition}

Given this definition, and following Section~\ref{sec:background}, we define the \acl{SR} $\widetilde{M}^\pi$ and the optimal state-action function $\widetilde{Q}^\star$ accordingly.
To measure the error that is introduced by the action repetition, we introduce the following definition:

\begin{definition}[Action-Repeat Value Error]
\label{def:repeat_error}
For a given repeat factor $k$, we define the \emph{action-repeat value error} as the worst-case discrepancy between the optimal Q-value function of the original MDP $\mathcal{M}$ and that of the action-repeat MDP $\widetilde{\mathcal{M}}$ as $\epsilon_\mathrm{repeat} (k) \coloneqq \| Q^\star - \widetilde{Q}^{\star} \|_\infty$.
\end{definition}

For the remainder of this paper, we assume the existence of a range of repeat factors $k$ such that the action-repeat value error $\epsilon_{\mathrm{repeat}}(k)$ grows slower than the spectral-error reduction. 
 Furthermore, all representations $(F, B)$ and successor measures $\hat{M}$ are hereafter assumed to be trained on the action-repeat \ac{MDP} $\widetilde{\mathcal{M}}$.

\subsection{Action Repetition Reduces the Optimality Gap}

In the following, we will highlight that repeating each action for $k$ steps introduces a trade-off between the action-repeat error $\epsilon_\mathrm{repeat}(k)$ and an accelerated spectral decay.
We denote the action-repeat policy-induced transition matrix by $\widetilde{P}^\pi \in \mathbb{R}^{|\mathcal{S} \times \mathcal{A}| \times |\mathcal{S} \times \mathcal{A}|}$ with entries $\widetilde{P}^\pi((s,a), (s', a')) = \widetilde{P}(s' \mid s, a)\,\pi(a' \mid s')$, and write $P_a \in \mathbb{R}^{|\mathcal{S}| \times |\mathcal{S}|}$ with entries $P_a(s, s') = P(s' \mid s, a)$ for the single-step state-transition matrix under a fixed action $a$.
We first combine our prior definitions to bound the overall approximation error. To derive this specific bound, we require the joint transition dynamics $\widetilde{P}^\pi$ to be diagonalizable, as stated in Assumption~\ref{assumption:diag_joint}.

\begin{restatable}[Spectral Bound of Optimality Gap for $k$-repeat FB Representations]{lemma}{OptGapRepeatLem}\label{lem:the_lemma}
Given a reward function $r\colon \mathcal{S} \times \mathcal{A} \to  \mathbb{R}$, let $\widetilde{r}(s, a) \coloneqq \mathbb{E}_P[\sum_{t=0}^{k-1} \gamma^t r(s_t, a) \mid s_0 = s]$ be the corresponding $k$-step expected reward function, and define $z_{\widetilde{r}} \coloneqq B^\top \widetilde{r}$.
Under Assumption~\ref{assumption:diag_joint}, let $(F, B)$ be an \ac{FB} representation with dimension $d$. Then the error in approximating the original optimal action-value function $Q^\star$ is bounded by:
\begin{equation}
 \left\| F(\,\cdot\,, \,\cdot\,, z_{\widetilde{r}})^\top z_{\widetilde{r}} - Q^{\star} \right\|_\infty
  \leq \epsilon_\mathrm{repeat}(k) + \frac{2 C_\mathrm{norm} \; \|r\|_\infty}{1 - \gamma} \left( \widetilde{\epsilon}_\mathrm{real}(\widetilde{r}) + \frac{C_{\mathrm{SF}}}{1 - \gamma^k |\lambda_{d+1}(\widetilde{P}^\pi)|}\right).
\end{equation}
\end{restatable}

Here $|\lambda_{d+1}(\widetilde{P}^\pi)|$ denotes the $(d{+}1)$-th largest absolute eigenvalue of $\widetilde{P}^\pi$, and $C_{\mathrm{SF}} > 0$ is a constant from the spectral truncation of the successor representation; since $\widetilde{P}^\pi$ is row-stochastic and $\gamma^k < 1$, the denominator is automatically positive.
Lemma~\ref{lem:the_lemma} captures the trade-off between repetition error $\epsilon_\mathrm{repeat}$, the \ac{FB} realization error $\widetilde{\epsilon}_\mathrm{real}$, and the spectral truncation controlled by $|\lambda_{d+1}(\widetilde{P}^\pi)|$.

As a next step, we examine how $|\lambda_{d+1}(\widetilde{P}^\pi)|$ changes depending on the number of repeats $k$. Specifically, we relate the spectrum of $\widetilde{P}^\pi$ to that of the per-action transition matrices $P_a$ collected into a single block-diagonal matrix $\mathbf{P}_\mathcal{A} \coloneqq \mathrm{diag}(P_{a_1}, \dotsc, P_{a_{|\mathcal{A}|}})$.
As stated in Assumption~\ref{assumption:diag_blocks}, we require the individual action blocks $P_a$ to be diagonalizable to cleanly bound the spectrum of matrix powers.

\begin{restatable}[Eigenvalue Contraction under Action Repetition]{lemma}{EigContractionLem}\label{lem:eig_contraction}
Under Assumption~\ref{assumption:diag_blocks}, the $(d{+}1)$-th largest absolute eigenvalue of $\widetilde{P}^\pi$ contracts exponentially in $k$:
\begin{equation}
|\lambda_{d+1}(\widetilde{P}^\pi)| \;\leq\; C_{\mathrm{rep}}\,|\lambda_{d+1}(\mathbf{P}_{\mathcal{A}})|^k.
\end{equation}
\end{restatable}

This establishes that the spectral term in Lemma~\ref{lem:the_lemma} contracts at exponential rate $|\lambda_{d+1}(\mathbf{P}_\mathcal{A})|^k$ whenever $|\lambda_{d+1}(\mathbf{P}_\mathcal{A})| < 1$, which requires $d \geq |\mathcal{A}|$ since $\mathbf{P}_\mathcal{A}$ has $|\mathcal{A}|$ unit eigenvalues.  The constant $C_{\mathrm{rep}} > 0$ is a worst-case bound that grows with $|\mathcal{S}|$ and $|\mathcal{A}|$.
The constants could potentially be tightened to $C_{\mathrm{rep}} = \sqrt{|\mathcal{A}|}$ and $C_{\mathrm{SF}} = 1$ under significantly stronger structural assumptions such as orthogonality.
Proofs as well as details on the constants and assumptions are provided in Appendix~\ref{sec:proofs}. Appendix~\ref{sec:tabular_bound} examines the spectral and representation terms in two discrete environments.

In practice, we believe the spectral decay of $\widetilde{P}^\pi$ to be significantly faster than suggested by the worst-case bound derived here.

\begin{figure}[tb]
    \centering
    \begin{subfigure}[b]{0.24\textwidth}
        \centering
        \includegraphics[width=0.99\textwidth]{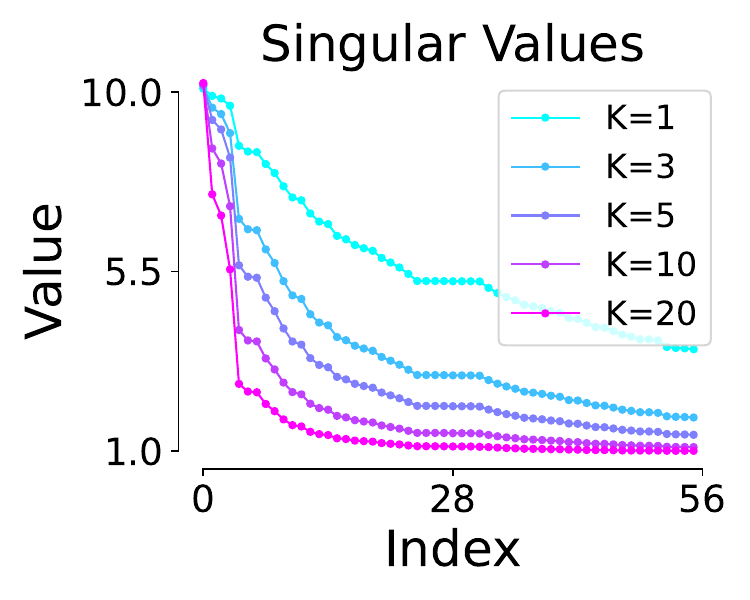}
        \label{fig:disc_sv_vary_k}
    \end{subfigure}
    \begin{subfigure}[b]{0.24\textwidth}
        \centering
        \includegraphics[width=0.99\textwidth]{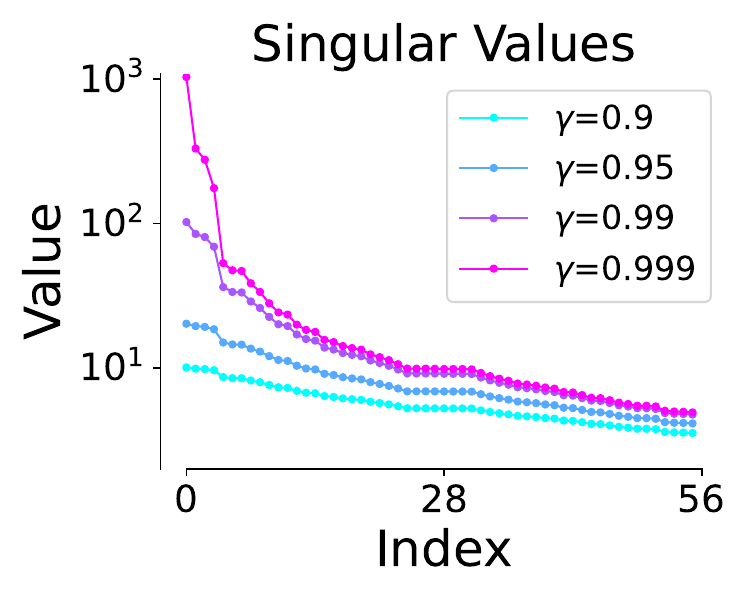}
        \label{fig:disc_sv_vary_gamma}
    \end{subfigure}
    \begin{subfigure}[b]{0.24\textwidth}
        \centering
        \includegraphics[width=0.99\textwidth]{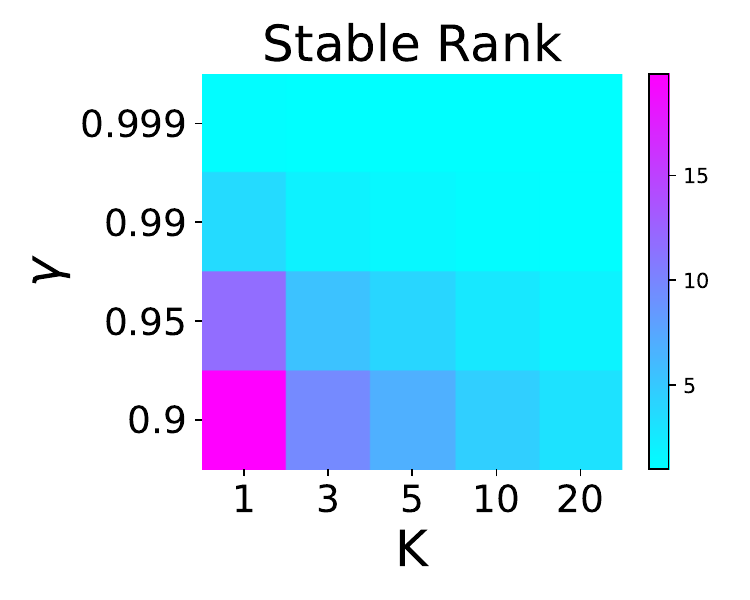}
        \label{fig:disc_srank}
    \end{subfigure}
    \begin{subfigure}[b]{0.24\textwidth}
        \centering
        \includegraphics[width=0.99\textwidth]{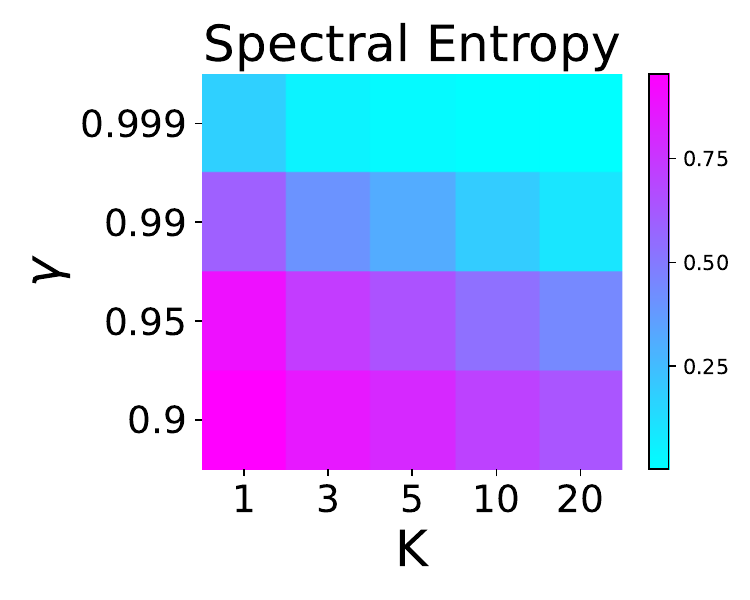}
        \label{fig:disc_ent}
    \end{subfigure}
    
    \begin{subfigure}[b]{0.24\textwidth}
        \centering
        \includegraphics[width=0.99\textwidth]{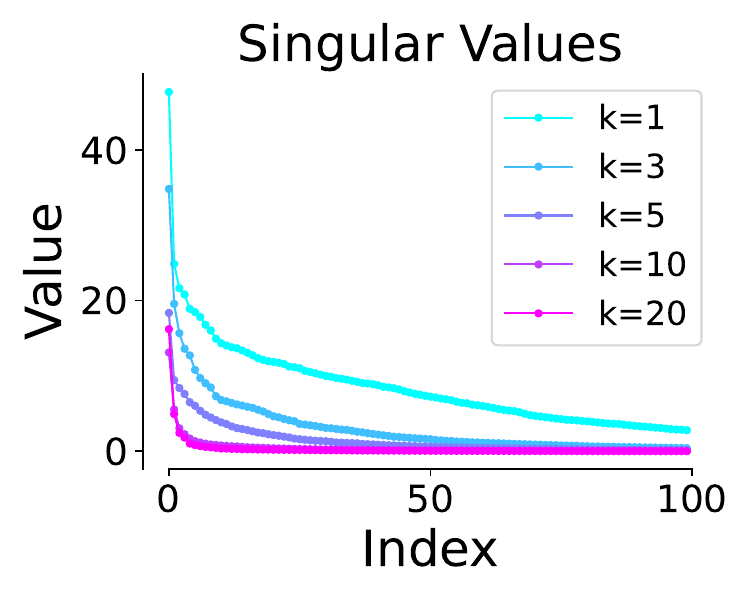}
        \label{fig:cont_sv_vary_k}
    \end{subfigure}
    \begin{subfigure}[b]{0.24\textwidth}
        \centering
        \includegraphics[width=0.99\textwidth]{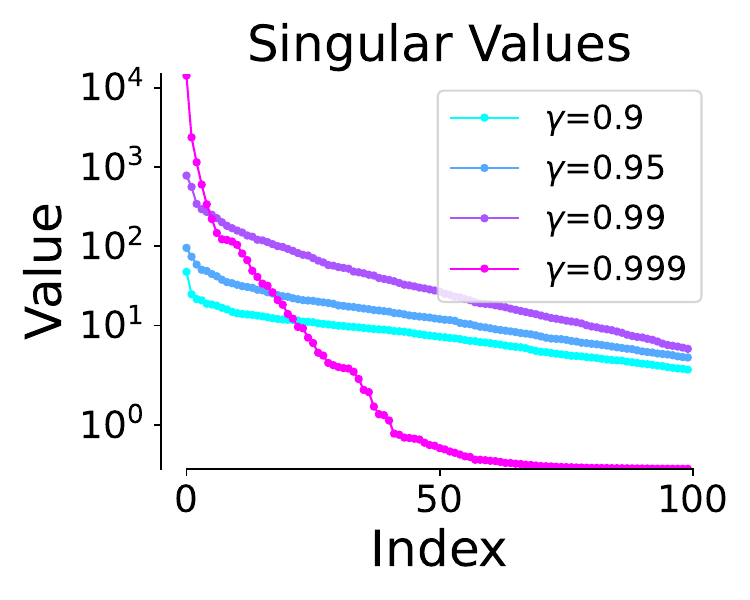}
        \label{fig:cont_sv_vary_gamma}
    \end{subfigure}
    \begin{subfigure}[b]{0.24\textwidth}
        \centering
        \includegraphics[width=0.99\textwidth]{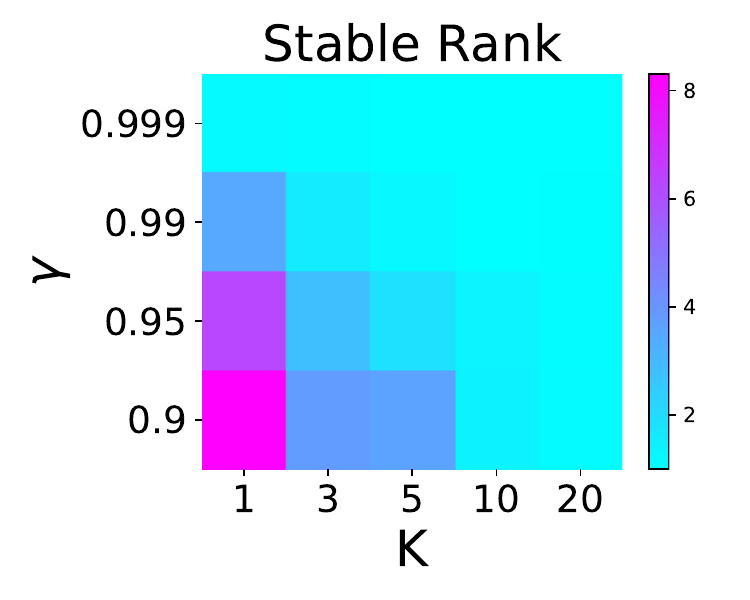}
        \label{fig:cont_srank}
    \end{subfigure}
    \begin{subfigure}[b]{0.24\textwidth}
        \centering
        \includegraphics[width=0.99\textwidth]{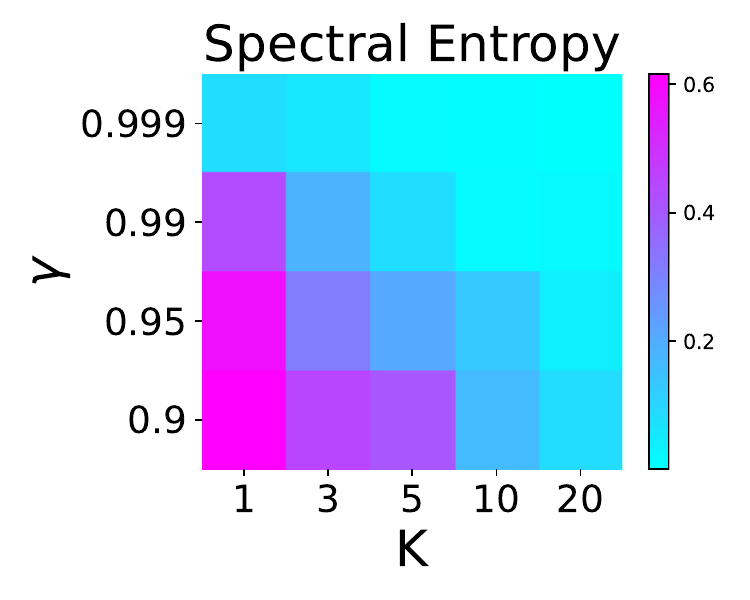}
        \label{fig:cont_ent}
    \end{subfigure}
    \caption{\textbf{Effect of temporal abstraction and discount factor on effective rank.} Effective rank decreases by increasing $k$ or $\gamma$ in both discrete \textit{(top)} and continuous \textit{(bottom)} Four-Rooms environment. Entropy decreases more smoothly as $k$ increases, suggesting a more stable reduction in effective rank as compared to increasing $\gamma$.}
    \label{fig:spectral_analysis}
\end{figure}

\section{Temporal Abstraction in Practice}
\label{sec:experiments}
We empirically validate the spectral insights from previous sections by examining how temporal abstraction shapes the structure and learnability of \ac{FB} representations. 
We introduce spectral metrics for the effective rank of the \ac{SR}, describe the experimental setup, and analyze how temporal abstraction reshapes the SR spectrum and affects performance. 
Finally, we study its interaction with embedding dimension and discount factor, highlighting their joint role in the stability and effectiveness of \ac{FB} representation learning.

\subsection{Spectral Metrics for Representation Complexity}
\label{subsec:spectral-metric}
To quantify the structure of the \ac{SR} and the effect of temporal abstraction, we use two complementary spectral metrics.

\paragraph{Stable Rank.}
Stable rank captures how much spectral energy is concentrated in dominant directions. 
It decreases when a few leading components dominate, making it a direct proxy for low-rank approximability.

\paragraph{Normalized Spectral Entropy.}
Normalized spectral entropy measures how evenly spectral energy is distributed.
High values indicate a diffused spectrum, while low values reflect concentration in a few components. 
Unlike stable rank, which emphasizes dominant modes, spectral entropy captures the overall spread of energy.

Together, these metrics characterize effective rank, distinguishing near rank-one collapse (low stable rank and entropy) from structured concentration, where few dominant components capture most energy while multiple modes remain active. Definitions of these metrics and details regarding their calculations for discrete and continuous settings are presented in Appendix~\ref{sec:spectral_metrics}.

\begin{figure*}[b]
    \centering
    \begin{subfigure}[b]{0.13\textwidth}
        \centering
        \adjustbox{width=1.0\textwidth, height=1.0\linewidth, cfbox=black 2pt 0cm}{\includegraphics[width=0.9\textwidth]{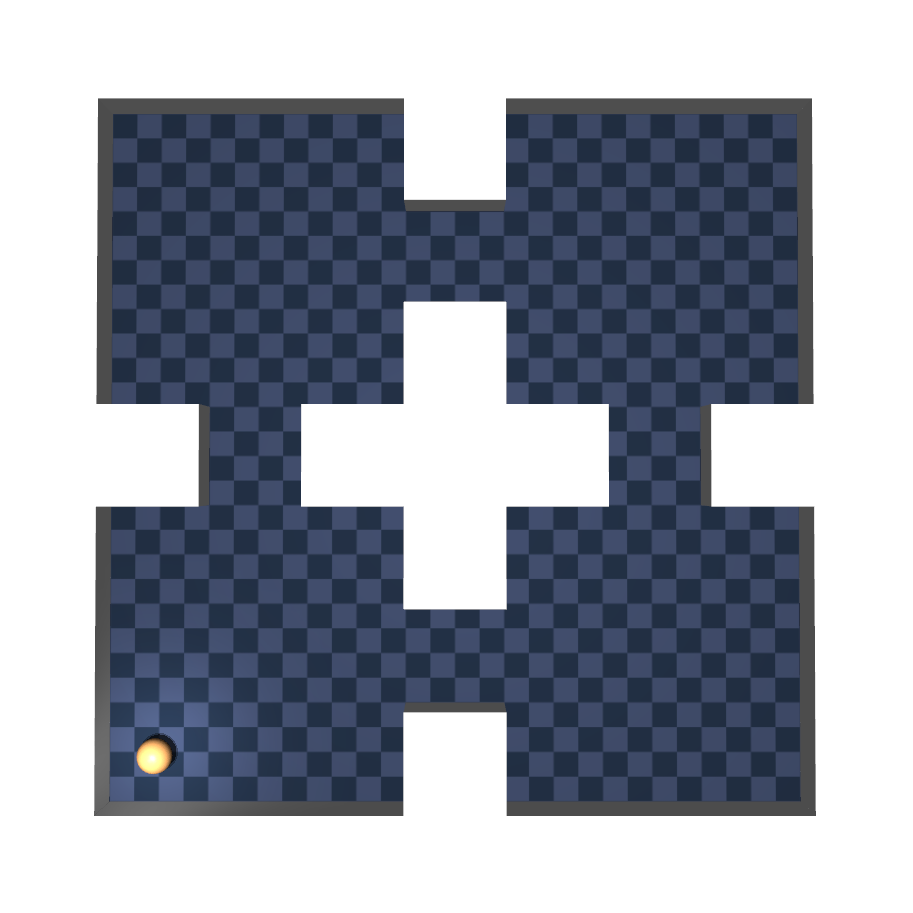}}
        \label{fig:rooms_environment}
    \end{subfigure}
    \hspace{0.3cm}
    \begin{subfigure}[b]{0.13\textwidth}
        \centering
        \adjustbox{width=1.0\textwidth, height=1.0\linewidth, cfbox=black 2pt 0cm}{\includegraphics[width=0.9\textwidth]{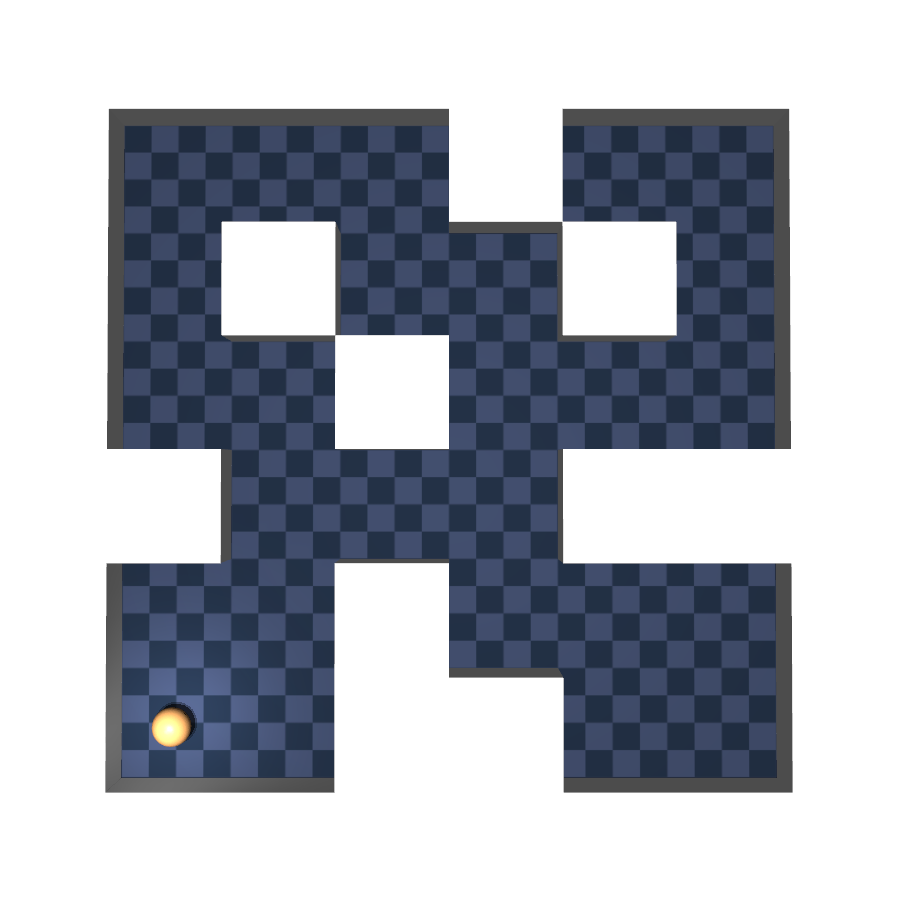}}
        \label{fig:maze_environment}
    \end{subfigure}
    \hspace{0.3cm}
    \begin{subfigure}[b]{0.168\textwidth}
        \centering
        \adjustbox{width=1.0\textwidth, height=0.775\linewidth, cfbox=black 2pt 0cm}{\includegraphics[width=0.9\textwidth]{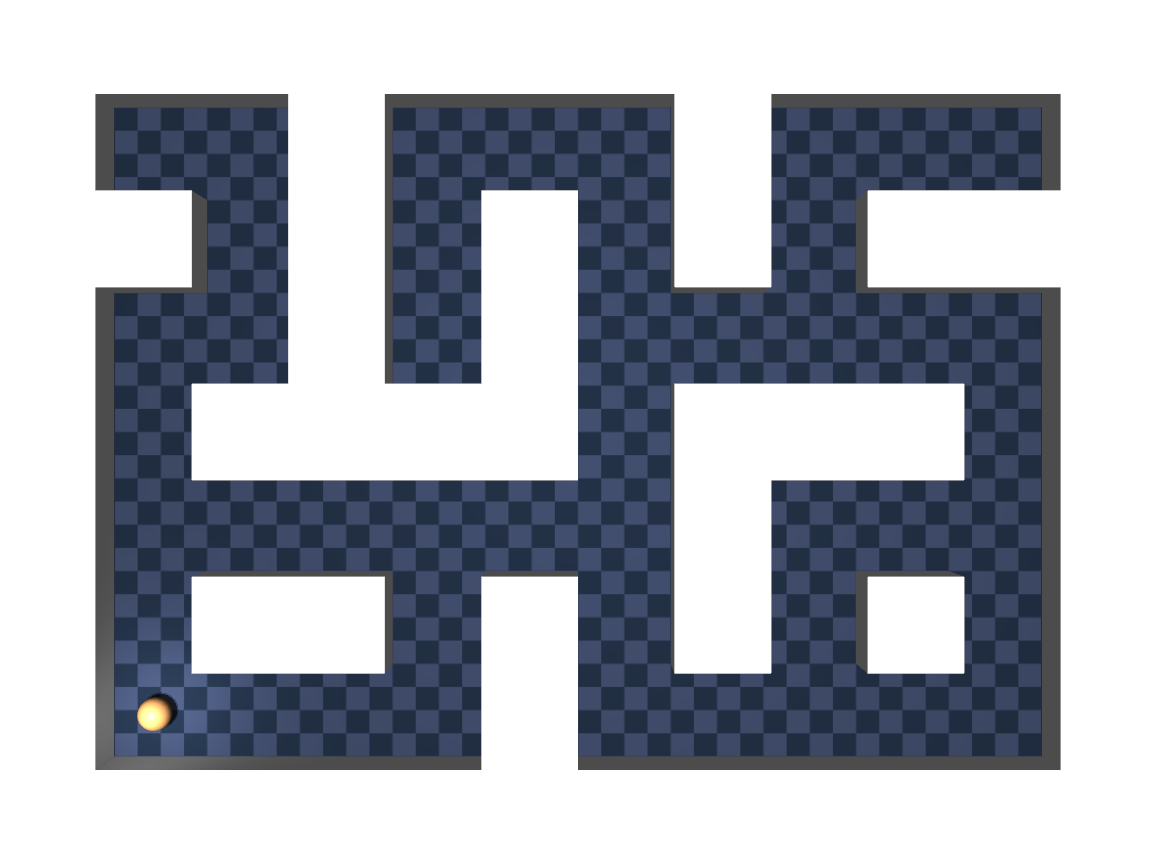}}
        \label{fig:largemaze_environment}
    \end{subfigure}
    \hspace{0.3cm}
    \begin{subfigure}[b]{0.13\textwidth}
        \centering
        \adjustbox{width=1.0\textwidth, height=1.0\linewidth, cfbox=black 2pt 0cm}{\includegraphics[width=0.9\textwidth]{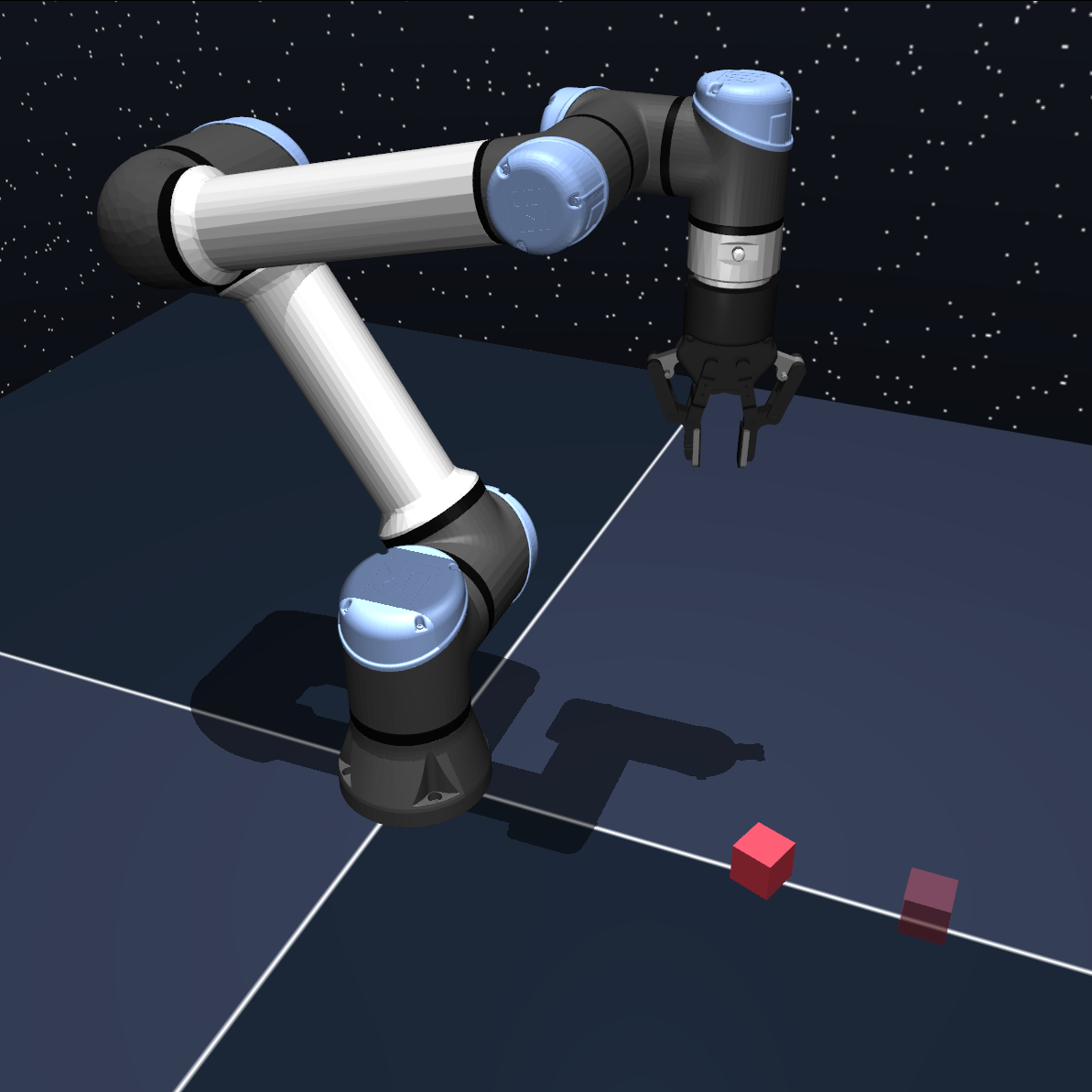}}
        \label{fig:cube_environment}
    \end{subfigure}
    \hspace{0.3cm}
    \begin{subfigure}[b]{0.13\textwidth}
        \centering
        \adjustbox{width=1.0\textwidth, height=1.0\linewidth, cfbox=black 2pt 0cm}{\includegraphics[width=0.9\textwidth]{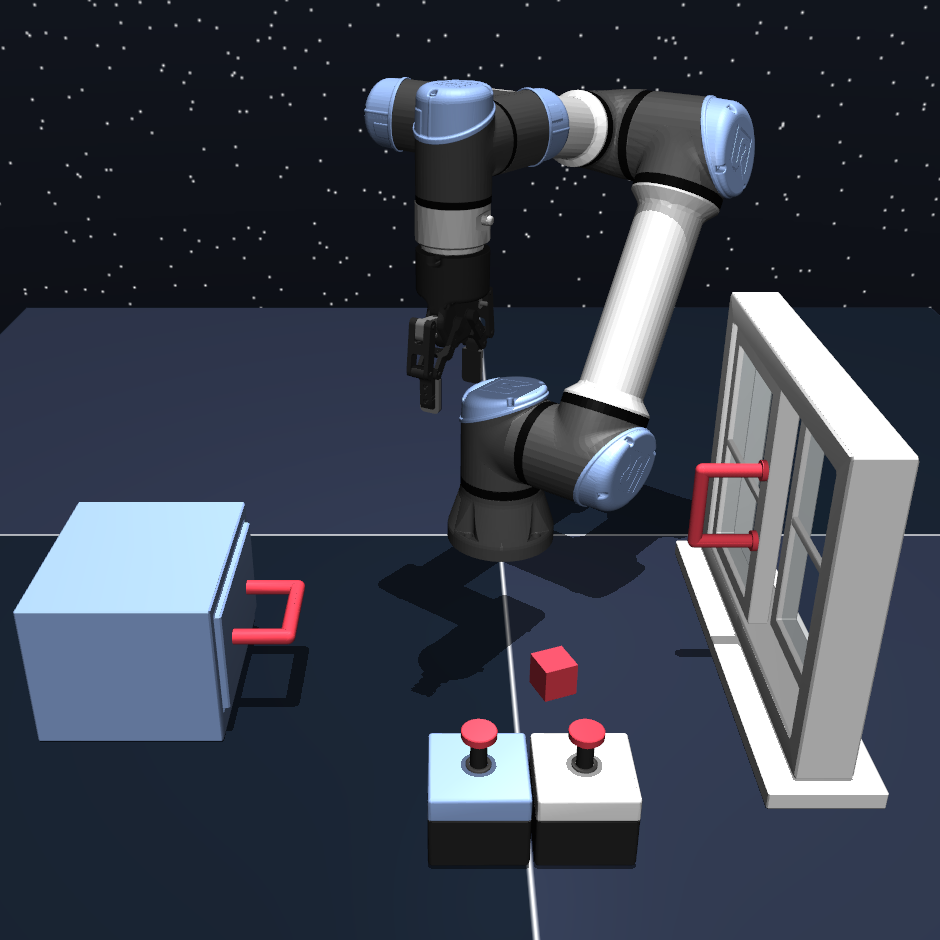}}
        \label{fig:scene_environment}
    \end{subfigure}
    \caption{\textbf{Continuous Navigation Environments:} \textit{Four-Rooms}, \textit{Maze}, \textit{Large-Maze}, \textit{Cube}, and \textit{Scene}}
    \label{fig:environments}
\end{figure*}

\subsection{Experimental Setup}
\label{subsec:exp_setup}

Throughout Sections~\ref{sec:experiments} and~\ref{sec:recipe} we follow practitioner usage: $\gamma$ refers to the \emph{nominal} discount used in training, i.e., the discount of the action-repeat MDP (corresponding to $\gamma^k$ in the notation of Definition~\ref{def:action_repeat_mdp}). The original-environment discount is then $\gamma_{\text{eff}} \coloneqq \gamma^{1/k}$. We refer to $\gamma_{\text{eff}}$ in Section~\ref{sec:recipe} when this distinction matters for analyzing horizon trade-offs.

We evaluate temporal abstraction via action repetition in three continuous maze navigation tasks of increasing difficulty and two manipulation tasks: \textit{Four-Rooms}, \textit{Maze}, \textit{Large-Maze}, \textit{Cube}, and \textit{Scene} (Figure~\ref{fig:environments}),  implemented in OGBench \citep{ogbench_park2025} with random start and goal positions.

Unless stated otherwise, experiments use Four-Rooms with discount factor $\gamma=0.95$, embedding dimension $d=100$, and action repetition $k=10$. 
Following \citet{touati2021learningone}, states are encoded from $(x,y)$ coordinates using an RBF kernel; similar results hold with learned CNN encoders (Figure~\ref{fig:input_modality}). 
We report the mean episodic return and the 95 percent confidence interval over five seeds. Each model is trained for one million gradient update steps. Detailed experiments and ablations on manipulation environments are presented in Appendix~\ref{sec:manipulation-exp}.
Key hyperparameters and implementation details are presented in Appendix~\ref{sec:hyperparameters}. 


\begin{figure}[t]
    \centering
    \begin{subfigure}[b]{0.26\textwidth}
        \centering
        \includegraphics[width=0.99\textwidth]{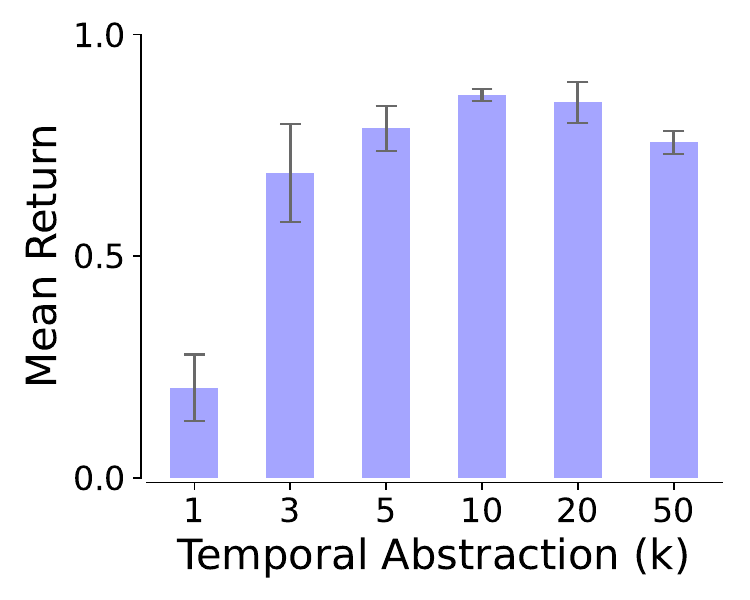}
        \caption{}
        \label{fig:k_ablation_barplot}
    \end{subfigure}
    \begin{subfigure}[b]{0.36\textwidth}
        \centering
        \includegraphics[width=0.99\textwidth]{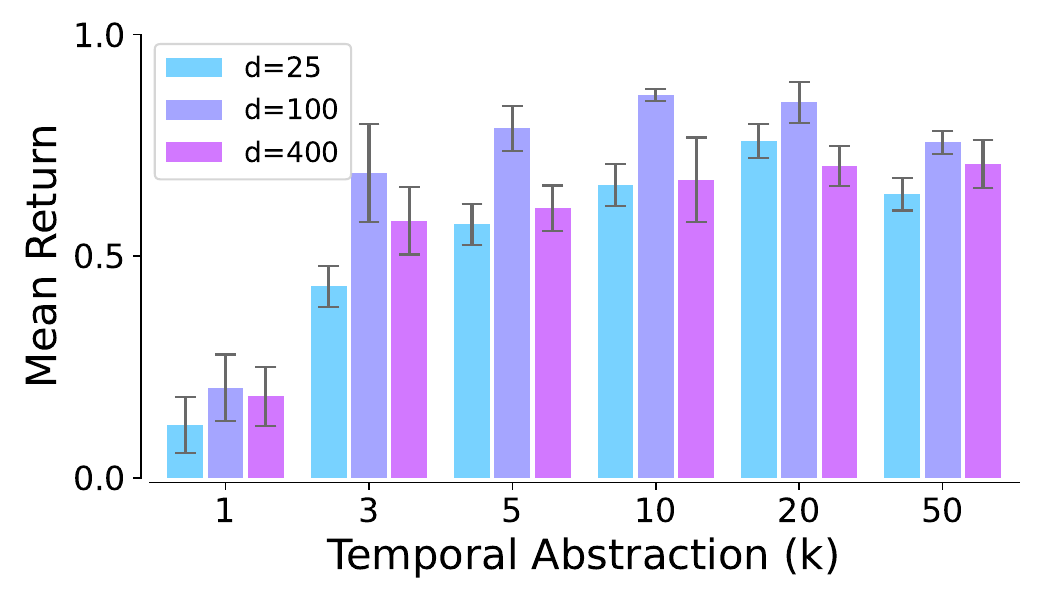}
        \caption{}
        \label{fig:k_z_ablation_barplot}
    \end{subfigure}
    \begin{subfigure}[b]{0.36\textwidth}
        \centering
        \includegraphics[width=0.99\textwidth]{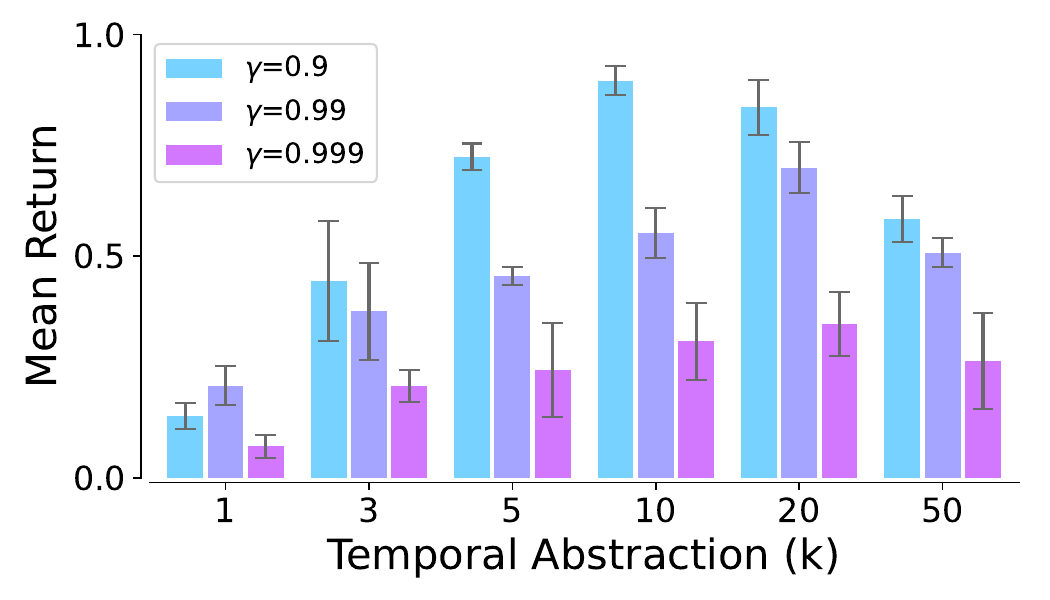}
        \caption{}
        \label{fig:k_gamma_ablation_barplot}
    \end{subfigure}
    \caption{\textbf{Effect of temporal abstraction on performance.} Ablation over temporal abstraction (k), embedding dimension (d), and discount factor ($\gamma$) using a continuous four-rooms environment. Addition of temporal abstraction ($k>1$) boosts performance, whereas increasing $d$ or $\gamma$ alone does not.}
    \label{fig:abtations_k_vs_z_vs_gamma}
\end{figure}

\subsection{Temporal Abstraction and Effective Rank}
Figure~\ref{fig:spectral_analysis} shows how increasing the temporal abstraction step $k$ reshapes the SR's singular value spectrum and reduces its effective rank. 
Across both discrete and continuous settings, larger $k$ accelerates the decay of tail singular values, concentrating energy in dominant components relevant for long-horizon control. 
This aligns the SR with the low-rank inductive bias of FB, improving representation quality.

However, excessive abstraction is detrimental. 
As stable rank and spectral entropy approach their minima, task-relevant dynamics are lost. 
This reflects the theoretical trade-off between spectral compression and bias from action repetition. 
Empirically, Figure~\ref{fig:k_ablation_barplot} shows performance degrading beyond an optimal $k$.

\subsection{Temporal Abstraction and Embedding Dimension of FB}
In principle, increasing the embedding dimension $d$ should improve SR approximation \citep{blier2021learning, touati2021learningone}. 
In continuous settings, however, this does not translate into better performance. 
Figure~\ref{fig:bellman_vary_z} shows that larger $d$ increases Bellman error, while Figure~\ref{fig:k_z_ablation_barplot} shows no performance gain without temporal abstraction ($k=1$), even when scaling $d$ from 25 to 400.

This behavior is consistent with our spectral analysis: higher capacity encourages fitting high-frequency components of a high-rank SR, which are hard to predict and amplify errors under bootstrapping. 
In contrast, temporal abstraction improves performance by reducing the effective rank of the target, simplifying the learning problem. 
Additional gains can be obtained by tuning $d$ once $k$ is fixed.

\begin{figure}[b]
    \centering
        \begin{subfigure}[b]{0.32\textwidth}
            \centering
            \includegraphics[width=0.9\textwidth]{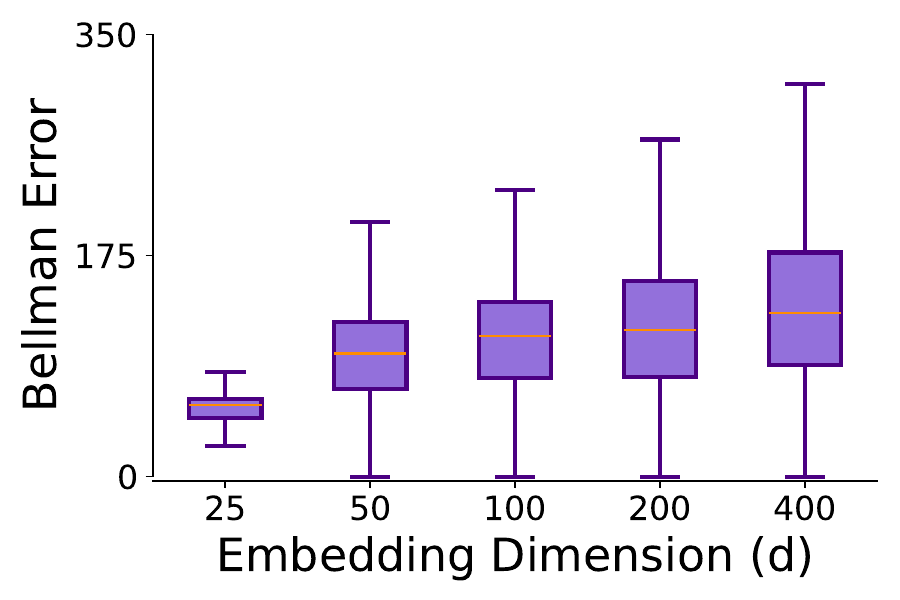}
            \caption{}
            \label{fig:bellman_vary_z}
        \end{subfigure}
        \begin{subfigure}[b]{0.32\textwidth}
            \centering
            \includegraphics[width=0.9\textwidth]{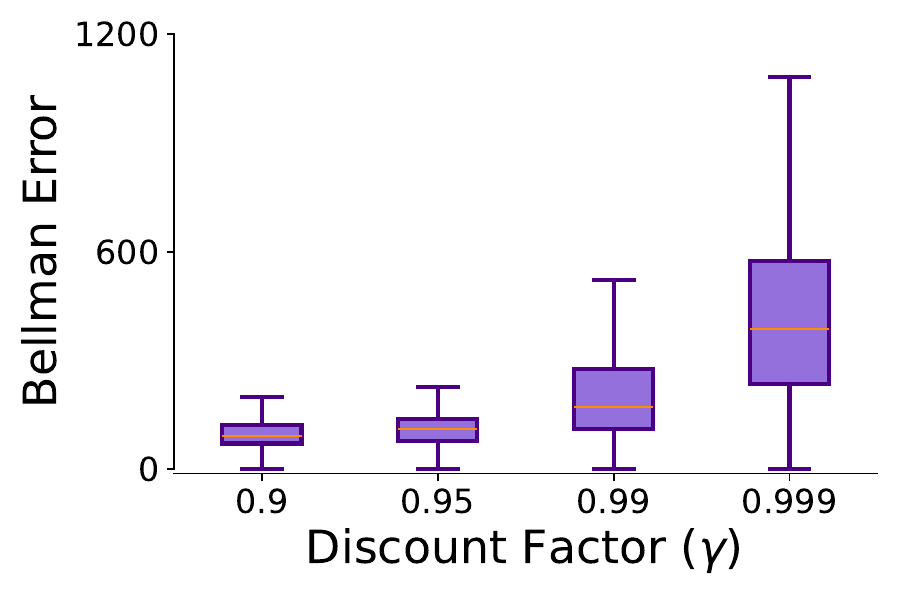}
            \caption{}
            \label{fig:bellman_vary_gamma}
        \end{subfigure}
        \begin{subfigure}[b]{0.32\textwidth}
            \centering
            \includegraphics[width=0.9\textwidth]{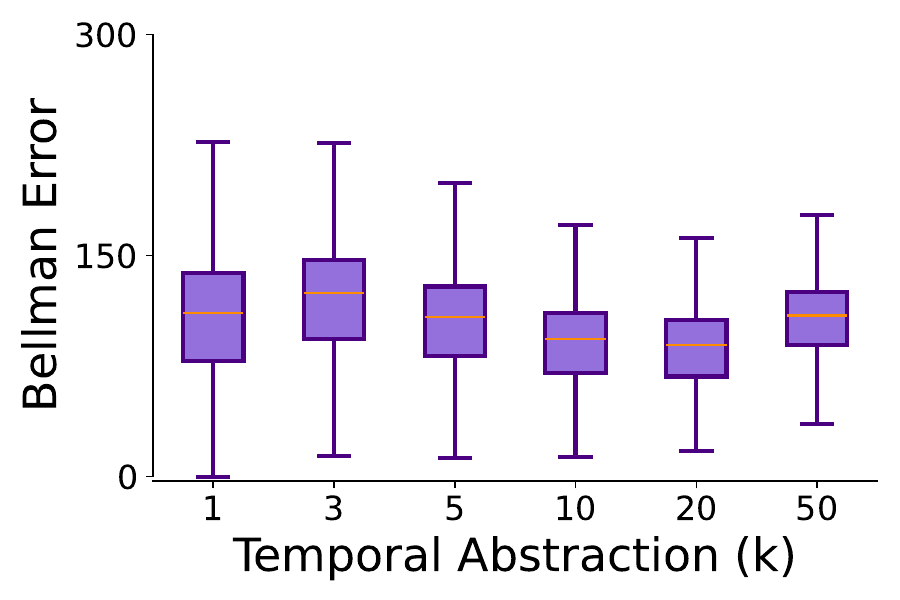}
            \caption{}
            \label{fig:bellman_vary_k}
        \end{subfigure}
    \caption{\textbf{Ablations: Bellman error.} Increasing the embedding dimension (a) or discount factor (b) without using temporal abstraction ($k=1$) leads to an increase in the Bellman error. Increasing $k$ (c) does not increase the Bellman error.}
    \label{fig:bellman_ablations}
\end{figure}
    
\begin{figure}[t]
    \centering
    \begin{subfigure}[b]{0.40\textwidth}
        \centering
        \includegraphics[width=0.85\textwidth]{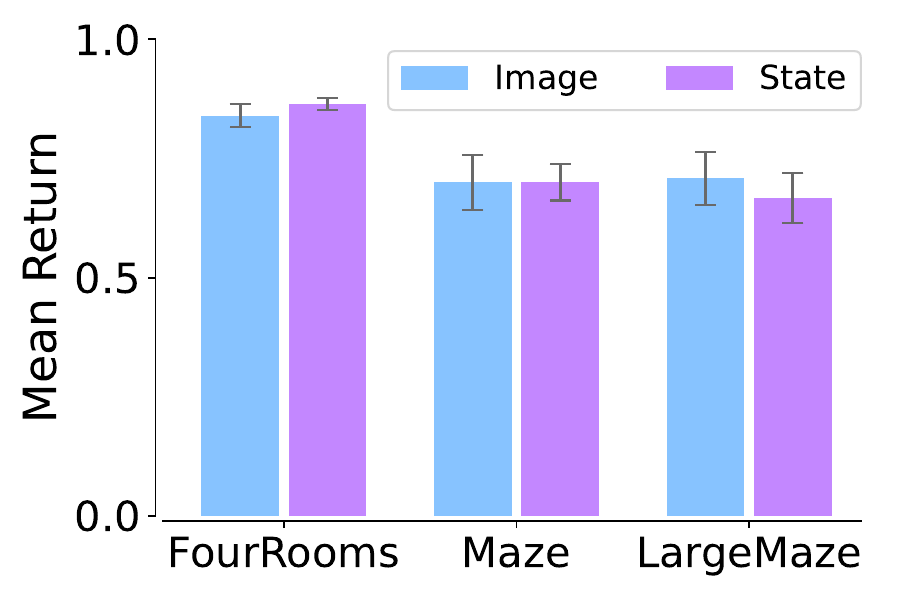}
        \caption{}
        \label{fig:input_modality}
    \end{subfigure}
    \begin{subfigure}[b]{0.40\textwidth}
        \centering
        \includegraphics[width=0.85\textwidth]{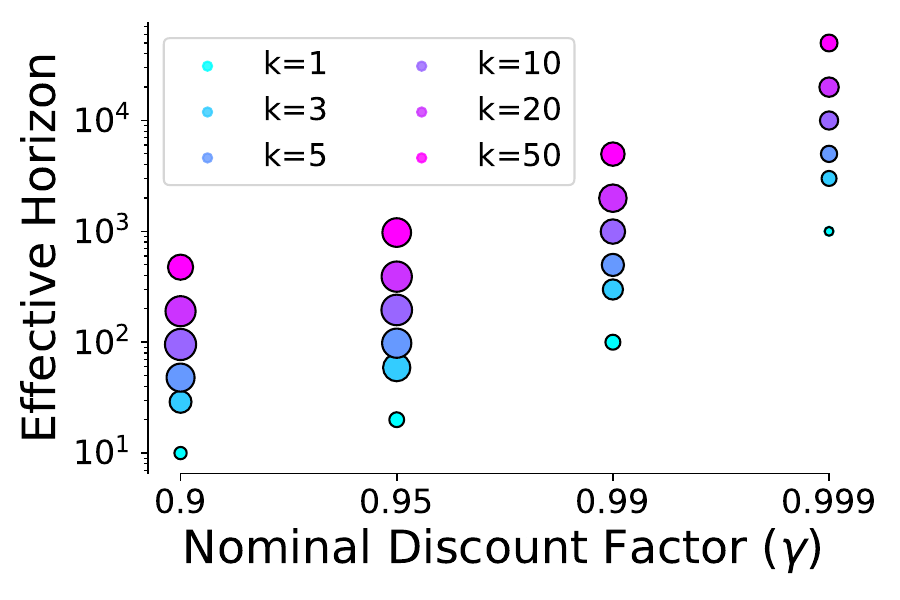}
        \caption{}
        \label{fig:effective_horizon}
    \end{subfigure}
    \caption{\textbf{Ablations: Input type and effective discount factor} (a): Image and State inputs use CNN and RBF encodings, respectively. (b): A higher return (larger radius) for a similar task horizon can be achieved by combining a lower nominal $\gamma$ with a higher $k$. The effective horizon (y-axis) is computed as $\frac{1}{(1-\gamma^{1/k})}$, expressing the repeat-MDP horizon in environment-frame units.}
    \label{fig:ablations_extra}
\end{figure}

\subsection{Spectral Dynamics: Discounting vs. Temporal Abstraction}
A low-rank \ac{SR} concentrates spectral energy in a small set of dominant components associated with long-horizon dynamics. 
As $\gamma \to 1$, these components are increasingly amplified, yielding stronger spectral concentration but also reduced training stability. 
This effect is reflected in Figure~\ref{fig:bellman_vary_gamma}, where the absolute Bellman error grows with $\gamma$.

While normalizing the Bellman residual by the magnitude of the $Q$-values reverses this trend by compensating for the scaling of SR values, the optimization dynamics are governed by the absolute residual. 
Consequently, the increase in absolute Bellman error at large $\gamma$ leads to higher gradient variance and a weaker effective contraction, which in turn degrades training stability. 
A more detailed comparison between relative and absolute Bellman errors is provided in Appendix~\ref{subsec:abs_vs_rel_bellman}.

Discounting and temporal abstraction modify the spectrum in fundamentally different ways. 
Increasing $\gamma$ amplifies existing components, including high-frequency ones, and degrades conditioning. 
In contrast, increasing $k$ smooths the dynamics by attenuating sub-dominant, high-frequency components while preserving the steady-state structure.

Although both reduce effective rank, their behavior differs sharply. 
Discount-driven compression is abrupt and unstable, often causing collapse in stable rank and entropy. 
Temporal abstraction instead induces a controlled spectral decay: the effective rank drops quickly for small $k$ and then stabilizes, maintaining higher spectral entropy. 
This enables a structured simplification of the predictive manifold without the instability of near-unity discounting.



\section{A Recipe for Effective Forward-Backward Representations}
\label{sec:recipe}
Our results show that optimal performance arises from combining moderate discounting with temporal abstraction, rather than simply maximizing $\gamma$.
Lower $\gamma$ improves stability but shortens the effective horizon; this can be compensated by increasing the action-repeat factor $k$.

We distinguish between the nominal discount $\gamma$ in the $k$-repeat \ac{MDP} and the effective discount $\gamma_{\text{eff}}$ in the original environment, related by $\gamma_{\text{eff}} = \gamma^{1/k}$. 
For a fixed task horizon, combining lower $\gamma$ with larger $k$ consistently outperforms standard high-discount settings (Figure~\ref{fig:effective_horizon}). 

Figure~\ref{fig:k_z_gamma_heatmap} further shows that moderate temporal abstraction ($k \in [5, 10]$) acts as a robust regularizer across embedding dimensions and discount factors. 
Although the optimal $k$ is environment-dependent, its inclusion yields consistent performance gains.

\begin{figure}[b]
    \centering
    \begin{subfigure}[b]{0.38\textwidth}
        \centering
        \includegraphics[width=0.9\textwidth]{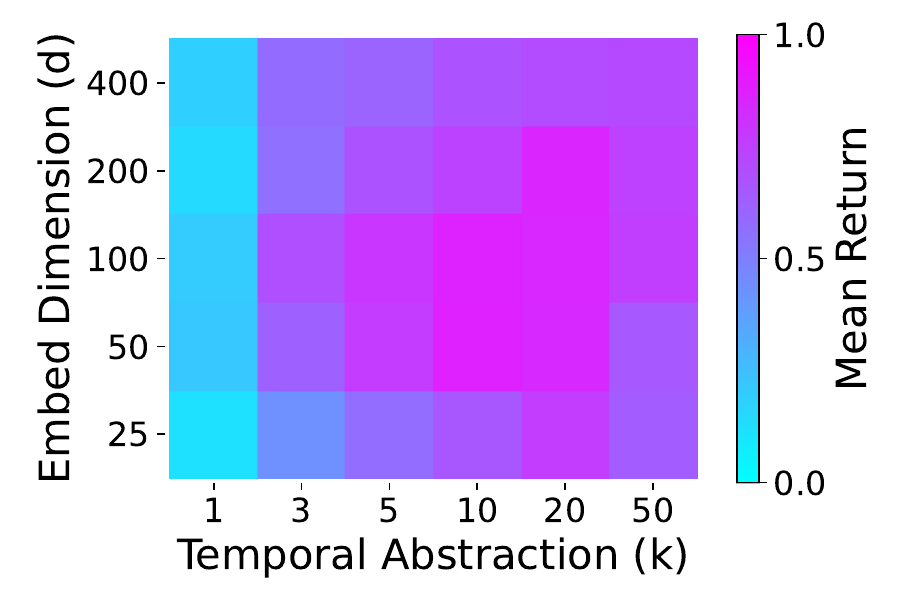}
        \caption{}
        \label{fig:diminish-zdim}
    \end{subfigure}
    \begin{subfigure}[b]{0.5\textwidth}
        \centering
        \includegraphics[width=0.9\textwidth]{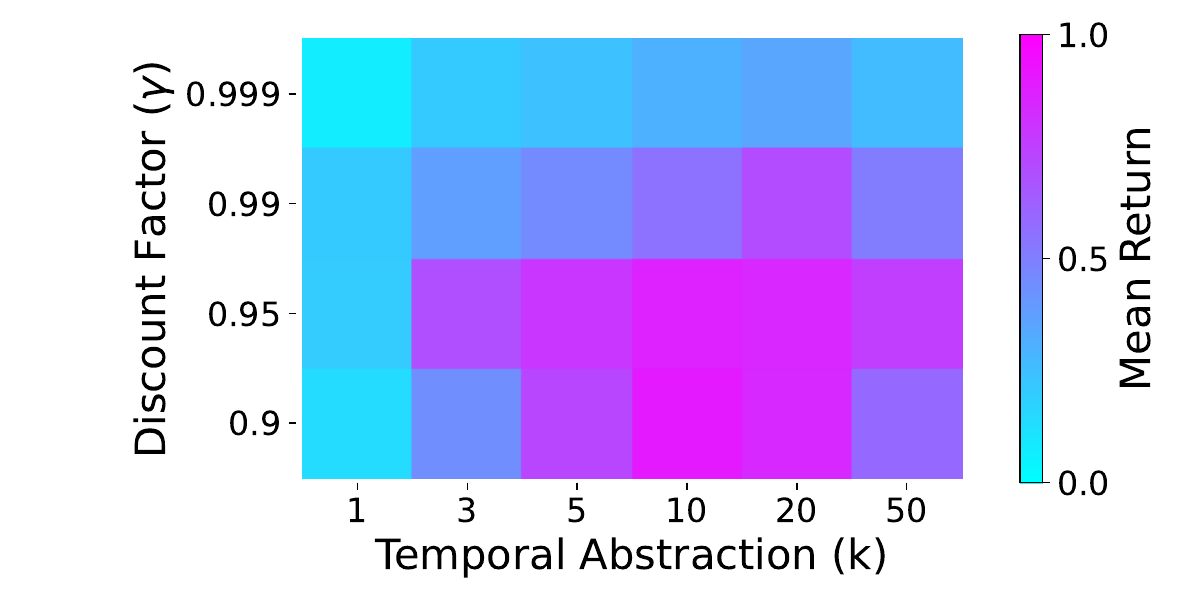}
        \caption{}
        \label{fig:diminish-gamma}
    \end{subfigure}
    \caption{\textbf{Ablation: Temporal Abstraction vs. Embedding Dimension vs. Discount Factor} Temporal abstraction offers a considerable boost in performance even with a moderate number of steps $k$. Performance is highest in the magenta region around $k\in[5,10]$ across moderate $\gamma$ and $d$. After the introduction of temporal abstraction, the performance is less sensitive to variations in the embedding dimension \textit{(a)} than to changes in the discount factor \textit{(b)}, especially as $\gamma \to 1$.}
    \label{fig:k_z_gamma_heatmap}
\end{figure}

\section{Conclusion}
\label{sec:conclusion}
We identify a key mismatch between the low-rank inductive bias of \acs{FB} and the inherently high-rank structure of the \acs{SR} in continuous domains. 
The SR exhibits a heavy spectral tail of high-frequency components that are difficult to approximate and amplify errors under bootstrapping.

We address this mismatch by introducing temporal abstraction as a spectral regulator. 
Action repetition acts analogously to a low-pass filter, attenuating high-frequency components while preserving steady-state dynamics, thereby reducing the effective rank of the SR. 
This yields a simpler and more learnable target for FB representations. 
Empirically, this spectral smoothing stabilizes learning and improves performance, even in high-discount regimes where standard FB struggles.

More broadly, our results suggest shifting focus from increasing model capacity to shaping the spectral structure of the underlying dynamics. 
Temporal abstraction serves as a practical tool for this purpose, enabling more stable and scalable predictive representations.

\section{Limitations and Future Work}
\label{sec:limitations}

While our study uses continuous maze navigation and object manipulation to isolate the spectral effects of temporal abstraction, several research avenues remain. 
First, the scope of our theoretical analysis is limited to finite dimensional state and action spaces. Although the findings are empirically verified in continuous settings a thorough examination of the theoretical findings in continuous settings is currently not present.

Second, we have primarily focused on action repetition as a fundamental form of temporal abstraction along with a preliminary comparison with Action Chunking~\citep{li2025reinforcement} in Appendix~\ref{sec:action-chunking}. More sophisticated frameworks, such as options or learned skills, may induce complex spectral transformations beyond the uniform attenuation studied here. Extending our analysis to adaptive abstractions could further clarify how hierarchical structures regularize representation learning.

Third, our results highlight an inherent trade-off between spectral stability and temporal resolution. As formalized in Definition~\ref{def:repeat_error}, the smoothing that facilitates tractable learning also introduces a bias that limits resolution of high-frequency dynamics. This approach may therefore be less suitable for tasks requiring near-instantaneous reactive control.

Finally, while we consider online interaction in moderate dimensions, scaling to high-dimensional observations or offline settings \citep{sikchi2025FastAW, tirinzoni2025ZeroShotWH} presents a compelling challenge. Investigating how temporal abstraction improves spectral conditioning in fixed datasets could significantly enhance the robustness of zero-shot generalization in offline reinforcement learning.
\subsubsection*{Acknowledgments}
\label{sec:ack}

We thank Raghu Rajan, Yannick Vogt, Jan Ole von Hartz, and Hanne Raum for insightful discussions and for their valuable feedback on the manuscript.
This work was funded by the Carl Zeiss Foundation through the ReScaLe project.
This work was partially funded by the German Research Foundation Emmy Noether Program grant number 468878300. 
\subsubsection*{Broader Impact}
\label{sec:impact}
This work advances the understanding of Forward--Backward representations, a general framework with potential applications across machine learning and robotics. 
We do not identify any immediate or specific societal risks beyond those broadly associated with these fields.
\newpage
\bibliography{main}
\bibliographystyle{rlj}
\newpage
\appendix
\setcounter{equation}{0}
\renewcommand{\theequation}{A\arabic{equation}}
\renewcommand{\theHequation}{A\arabic{equation}}   

\section{Hyperparameters and Implementation Details}
\label{sec:hyperparameters}
We summarize the key hyperparameters used for training the \acf{FB} representation in Table~\ref{tab:hyperparameters}. 
These settings were kept fixed across experiments unless stated otherwise. 
The $k$-step action repetition is implemented via a wrapper over the environment. In other words, the agent will only interact with the $k$-repeat MDP and will not have access to the intermediate observations among the $k$ steps.
We use episodic return as a measure of performance for the agents. 
All environments provide a sparse (zero or one) reward.
To get the final performance or on each validation step, each model is evaluated on 50 episodes with random start and goal resets.
The latent vector $z$ describing the task or goal is sampled during training using a 50-50 mix of sampling from a normal distribution or from a random visited state in the replay buffer projected to the latent space using the backward network, similarly to \citet{tirinzoni2025ZeroShotWH}.

\begin{table}[htbp]
    \caption{Relevant hyperparameters used for training \ac{FB} representation.}
    \begin{center}
        \begin{tabular}{ll}
            \\ \hline \\
            Training steps (gradient update) & 1e6 \\
            Reward type & Sparse \\
            Hidden Layers (Forward, Backward, Actor) & [256, 256] \\
            Ensemble (Forward) & 2 \\
            Learning Rate (Backward, Actor) & 1e-6 \\
            Learning Rate (Forward) & 1e-5 \\
            Batch size (state) & 512 \\
            Batch size (image) & 128 \\
            Replay Buffer size & 1e6 \\
            FB Orthogonal Loss Coef. & 1.0 \\
            z-latent: buffer data vs. random sampling ratio & 0.5 \\
            z-latent: hold steps before resampling & 10 \\  
        \end{tabular}
    \end{center}
    \label{tab:hyperparameters}
\end{table}

\newpage
\section{Proofs}
\label{sec:proofs}

In the following, we provide the proofs of this work.
Note, that in \cite{touati2021learningone}, the reward embedding $z_R \coloneqq  B^\top r \nu$ is weighted by a data distribution $\nu$. For this work, we assume a uniform distribution and implicitly absorb its normalization constant into the scaling of $B$, simplifying the embedding to the matrix-vector product $z_R = B^\top r$.

\subsection{Optimality Gap for FB Representations}
The following theorem is a simplified refinement \cite[Theorem 8]{touati2021learningone} for our spectral analysis setting.

\OptGapThm*

\begin{proof}
Applying \cite[Theorem 8]{touati2021learningone} to our setting we have
\begin{equation}
    \left\| F(\,\cdot\,, \,\cdot\,,\ z_R)^\top z_R - Q^{\star} \right\|_\infty \leq \frac{2 \|r\|_A}{(1 - \gamma)} \sup_{s, a}\| \hat{M}^{z_R}(s, a, \,\cdot\,, \,\cdot\,) - M^{\pi_{z_R}}(s, a, \,\cdot\,, \,\cdot\,) \|_B,
\end{equation}
where the norms $\|\cdot\|_A$ on functions and $\|\cdot\|_B$ on (signed) measures must satisfy the duality compatibility $|\langle f, \mu \rangle| \leq \|f\|_A \|\mu\|_B$ for all $f, \mu$.
Choosing $\|\cdot\|_A = \|\cdot\|_\infty$ and $\|\cdot\|_B = \|\cdot\|_2$, H\"older's inequality combined with $\|\mu\|_1 \leq \sqrt{|\mathcal{S}||\mathcal{A}|}\,\|\mu\|_2$ yields $|\langle f, \mu\rangle| \leq \|f\|_\infty \|\mu\|_1 \leq C_{\mathrm{norm}}\,\|f\|_\infty\,\|\mu\|_2$ with $C_{\mathrm{norm}} = \sqrt{|\mathcal{S}||\mathcal{A}|}$. Given this, we have
\begin{align}
    \left\| F(\,\cdot\,, \,\cdot\,,\ z_R)^\top z_R - Q^{\star} \right\|_\infty 
    & \leq \frac{2\,C_{\mathrm{norm}}\,\|r\|_\infty}{(1 - \gamma)} \sup_{s, a}\| \hat{M}^{z_R}(s, a, \,\cdot\,, \,\cdot\,) - M^{\pi_{z_R}}(s, a, \,\cdot\,, \,\cdot\,) \|_2\\
    & \leq \frac{2\,C_{\mathrm{norm}}\,\|r\|_\infty}{(1 - \gamma)} \| \hat{M}^{z_R} - M^{\pi_{z_R}} \|_2.
\end{align}
This ends the proof.
\end{proof}

%
%
%
%
%
%


\subsection{Spectral Bound for Optimality Gap of k-repeat FB Representations}

We first derive the matrix form of the action-repeat policy-induced transition matrix used by both Lemma~\ref{lem:the_lemma} and Lemma~\ref{lem:eig_contraction}. 

For a fixed action $a_q$, let $P_{a_q} \in \mathbb{R}^{|\mathcal{S}| \times |\mathcal{S}|}$ denote the state-transition dynamics under that action, and let $\pi_{s_p} \in \mathbb{R}^{1 \times |\mathcal{A}|}$ represent the row vector of policy probabilities for a given state $s_p$, i.e.,
\begin{equation}
P_{a_q} = 
\begin{bmatrix}
P(s_1 \mid s_1, a_q) & \dots & P(s_{|\mathcal{S}|} \mid s_1, a_q) \\
\vdots & \ddots & \vdots \\
P(s_1 \mid s_{|\mathcal{S}|}, a_q) & \dots & P(s_{|\mathcal{S}|} \mid s_{|\mathcal{S}|}, a_q)
\end{bmatrix},\quad
\pi_{s_p} = 
\begin{bmatrix}
\pi(a_1 \mid s_p) & \dots & \pi(a_{|\mathcal{A}|} \mid s_p)
\end{bmatrix}.
\end{equation}
Stacking the per-action matrices into a block-diagonal matrix $\mathbf{P}_{\mathcal{A}}$, we can express the transition matrix $\widetilde{P}^\pi \in \mathbb{R}^{|\mathcal{S} \times \mathcal{A}| \times |\mathcal{S} \times \mathcal{A}|}$ as a product of action-repetition and policy-mapping components:
\begin{equation}\label{eq:til_ppi_comp}
\widetilde{P}^\pi =  K\,\mathbf{P}_{\mathcal{A}}^k\,E\,\tilde{\pi},
\end{equation}
where
\begin{equation*}
\mathbf{P}_{\mathcal{A}}^k =
\begin{bmatrix}
P_{a_1}^k & & 0 \\
& \ddots & \\
0 & & P_{a_{|\mathcal{A}|}}^k
\end{bmatrix} \in \mathbb{R}^{|\mathcal{S} \times \mathcal{A}| \times |\mathcal{S} \times \mathcal{A}|}
\quad\mbox{and}\quad
\tilde{\pi} = 
\begin{bmatrix}
\pi_{s_1} & & 0 \\
& \ddots &  \\
0 & & \pi_{s_{|\mathcal{S}|}}
\end{bmatrix} \in \mathbb{R}^{|\mathcal{S}| \times |\mathcal{S} \times \mathcal{A}|}.
\end{equation*}
The matrix $K \in \mathbb{R}^{|\mathcal{S} \times \mathcal{A}| \times |\mathcal{S} \times \mathcal{A}|}$ is a commutation matrix that reorders the state-action product space from a state-major to an action-major indexing scheme, and $E \in \mathbb{R}^{|\mathcal{S} \times \mathcal{A}| \times |\mathcal{S}|}$ is a broadcasting matrix that lifts a vector from $\mathbb{R}^{|\mathcal{S}|}$ to $\mathbb{R}^{|\mathcal{S} \times \mathcal{A}|}$ by replicating each state coordinate $|\mathcal{A}|$ times.
In the decomposition \eqref{eq:til_ppi_comp}, $\mathbf{P}_{\mathcal{A}}^k$ captures the $k$-step transitions under action repetition, while $\widetilde{\pi}$ maps the policy within the state-action space.
Intuitively, the system first evolves for $k$ steps under the same action, after which the next action is selected according to the policy without execution.

Before stating the lemma we introduce a diagonalizability assumption on the joint chain.

\begin{assumption}[Diagonalizability of joint dynamics]
\label{assumption:diag_joint}
The policy-induced transition matrix $\widetilde P^\pi$ of the action-repeat MDP is
diagonalizable over $\mathbb{C}$, i.e.\ $\widetilde P^\pi = S \Lambda S^{-1}$ for some
invertible matrix $S$.
\end{assumption}

This is a standard assumption in spectral analyses of Markov chains and matrix perturbation theory \citep{horn2013matrix, stewart1990matrix} and is generic: matrices with distinct eigenvalues are dense in $\mathbb{R}^{n \times n}$ \citep[Theorem \ 2.4.7.1]{horn2013matrix}. We therefore expect it to hold in essentially most discrete environments of practical interest, since diagonalizability fails only when two or more eigenvalues coincide and their eigenvectors fail to span the corresponding joint eigenspace, an exact algebraic degeneracy broken by any stochasticity or asymmetry in the transition dynamics. The tightness of the bound is controlled by $\kappa(S) = \|S\|_2\|S^{-1}\|_2$, which equals $1$ when $\widetilde P^\pi$ is normal and grows as $\widetilde P^\pi$ approaches a defective matrix.



\OptGapRepeatLem*

\begin{proof}
The proof bounds the spectral truncation error of the action-repeat successor representation $\widetilde{M}^\pi$ in terms of $|\lambda_{d+1}(\widetilde{P}^\pi)|$ (Step~1), the $k$-step reward magnitude (Step~2), and combines these into the final bound (Step~3).

\textbf{Step 1: Spectral bound of the discounted infinite horizon.}
The successor representation $\widetilde{M}^\pi$ for the action-repeat MDP is defined by the discounted sum of future transitions $\widetilde{M}^\pi \coloneqq \sum_{t=0}^{\infty} (\gamma^k \widetilde{P}^\pi)^t$. Since $\widetilde{P}^\pi$ is a row-stochastic matrix, its spectral radius satisfies $\rho(\widetilde{P}^\pi) = 1$ and thus with discounting we have $\rho(\gamma^k \widetilde{P}^\pi) = \gamma^k < 1$. Thus, the Neumann series converges to the matrix inverse:
\begin{equation}\label{eq:neumann_series}
    \widetilde{M}^\pi = (I - \gamma^k \widetilde{P}^\pi)^{-1}.
\end{equation}

Since $\widetilde{P}^\pi = S \Lambda S^{-1}$ (Assumption~\ref{assumption:diag_joint}), the successor representation of the action-repeat MDP (Definition~\ref{def:action_repeat_mdp}, with discount $\gamma^k$) is
\begin{equation}
\widetilde{M}^\pi = (I - \gamma^k \widetilde{P}^\pi)^{-1} = S\,(I - \gamma^k \Lambda)^{-1}\, S^{-1} = S\,\mathrm{diag}\!\left(\frac{1}{1 - \gamma^k \lambda_i}\right) S^{-1}.
\end{equation}

Let $\Phi = (I - \gamma^k \Lambda)^{-1} = \mathrm{diag}(1/(1 - \gamma^k \lambda_i))$. 
By submultiplicativity of singular values \citep{horn2013matrix}:
\begin{equation}
\sigma_{d+1}(\widetilde{M}^\pi) = \sigma_{d+1}(S \Phi S^{-1}) 
  \leq \|S\|_2 \cdot \sigma_{d+1}(\Phi) \cdot \|S^{-1}\|_2 
  = \kappa(S) \cdot \sigma_{d+1}(\Phi).
\end{equation}
Let $\lambda_i = \lambda_i(\widetilde{P}^\pi)$ be the $i$-th eigenvalue of 
$\widetilde{P}^\pi$, ordered by nonincreasing modulus. Since $\Phi$ is diagonal, its 
singular values are the moduli of its diagonal entries, which we denote $\phi_i \coloneqq 1/|1 - \gamma^k \lambda_i|$ for $i=1,\dotsc,n$.
Note that, as $\lambda_i$ may be complex the sequence of $\phi_i$ is not sorted by modulus and we have in general $\sigma_i(\Phi) \neq \phi_i$ but $\sigma_j(\Phi) = [(j\text{-th largest} \{ \phi_i\}]$.
Since $\widetilde{P}^\pi$ is row-stochastic we have $|\lambda_i| \leq 1$, so $\gamma^k |\lambda_i| < 1$. Applying the reverse triangle inequality to each denominator gives the pointwise bound
\begin{equation}
\phi_i = \frac{1}{|1 - \gamma^k \lambda_i|} \;\leq\; \frac{1}{1 - \gamma^k |\lambda_i|} \;=:\; \tilde{\phi}_i,
\end{equation}
where the dominating sequence $\tilde{\phi}_i$ is monotone increasing in $|\lambda_i|$. Since $|\lambda_i|$ is sorted nonincreasingly, the $(d{+}1)$-th largest element of $\{\tilde{\phi}_i\}_{i=1}^n$ is exactly $\tilde{\phi}_{d+1} = 1/(1-\gamma^k|\lambda_{d+1}|)$.

Combining the pointwise bound $\phi_i \leq \widetilde\phi_i$ with the monotonicity of $\widetilde\phi_i$ in $|\lambda_i|$, we obtain
\begin{equation}
\sigma_{d+1}(\Phi) 
  \;= \big[ (d{+}1)\text{-th largest of } {\{\phi_i\}} \big] \;\leq\; \big[(d{+}1)\text{-th largest of } \{\widetilde\phi_i\}\big]
  \;=\; \widetilde\phi_{d+1}
  \;=\; \frac{1}{1-\gamma^k|\lambda_{d+1}|}.
\end{equation}

Combining with the submultiplicativity bound, the SVD truncation error can be bounded by
\begin{equation*}
\sigma_{d+1}(\widetilde{M}^\pi) 
  \leq \kappa(S) \cdot \sigma_{d+1}(\Phi) 
  \leq \frac{\kappa(S)}{1 - \gamma^k |\lambda_{d+1}(\widetilde{P}^\pi)|}.
\end{equation*}

\textbf{Step 2: Bound reward $\widetilde{r}$.}
By definition, the expected $k$-step reward is given by $\widetilde{r}(s, a) = \mathbb{E}_P\left[\sum_{t=0}^{k-1} \gamma^t r(s_t, a) \mid s_0 = s\right]$. 
Bounding the reward at each step with $\|r\|_\infty$, we get:
\begin{equation}\label{eq:reward_bound}
\|\widetilde{r}\|_\infty \leq \sum_{t=0}^{k-1} \gamma^t \|r\|_\infty = \frac{1 - \gamma^k}{1 - \gamma} \|r\|_\infty.
\end{equation}

\textbf{Step 3: Complete the proof.}
Using the triangle inequality and Definition~\ref{def:repeat_error}, we first separate the action-repeat approximation error:
\begin{align}
 \left\| F(\,\cdot\,, \,\cdot\,, z_{\widetilde{r}})^\top z_{\widetilde{r}} - Q^{\star} \right\|_\infty
  & \leq \left\| F(\,\cdot\,, \,\cdot\,, z_{\widetilde{r}})^\top z_{\widetilde{r}} - \widetilde{Q}^{\star} \right\|_\infty + \| \widetilde{Q}^\star - Q^\star \|_\infty \nonumber \\
  & = \left\| F(\,\cdot\,, \,\cdot\,, z_{\widetilde{r}})^\top z_{\widetilde{r}} - \widetilde{Q}^{\star} \right\|_\infty + \epsilon_\mathrm{repeat}(k). 
\intertext{Next, applying Theorem~\ref{theorem:opt_gap} to the first term and substituting our reward bound from Step 2:}
  & \leq \epsilon_\mathrm{repeat}(k) + \frac{2 \,C_\mathrm{norm}\, \|\widetilde{r}\|_\infty}{1 - \gamma^k} \| \hat{M}^{z_{\widetilde{r}}} - \widetilde{M}^{\pi_{z_{\widetilde{r}}}} \|_2 \nonumber \\
  & \leq \epsilon_\mathrm{repeat}(k) + \frac{2 \,C_\mathrm{norm}\, \|r\|_\infty}{1 - \gamma} \| \hat{M}^{z_{\widetilde{r}}} - \widetilde{M}^{\pi_{z_{\widetilde{r}}}} \|_2.
\intertext{Finally, bounding the model error by the realizability error (Definition~\ref{def:real_error}) and the spectral truncation properties established in Step~1:}
  & \leq \epsilon_\mathrm{repeat}(k) + \frac{2 \,C_\mathrm{norm}\, \|r\|_\infty}{1 - \gamma} \left( \widetilde{\epsilon}_\mathrm{real}(\widetilde{r}) + \sigma_{d + 1}(\widetilde{M}^{\pi_{z_{\widetilde{r}}}}) \right) \nonumber \\
  & \leq \epsilon_\mathrm{repeat}(k) + \frac{2 \,C_\mathrm{norm}\, \|r\|_\infty}{1 - \gamma} \left( \widetilde{\epsilon}_\mathrm{real}(\widetilde{r}) + \frac{\kappa(S)}{1 - \gamma^k |\lambda_{d+1}(\widetilde{P}^\pi)|}\right).
\end{align}
Defining $C_\mathrm{SF} \coloneqq \kappa(S)$ completes the proof.
\end{proof}

\subsection{Eigenvalue Contraction under Action Repetition}
\label{subsec:eig_contraction}
Lemma~\ref{lem:the_lemma} bounds the truncation error in terms of $|\lambda_{d+1}(\widetilde P^\pi)|$, but does not say how this eigenvalue depends on the action-repeat horizon $k$. We now show that, under a structural condition on the per-action transition matrices, $|\lambda_{d+1}(\widetilde P^\pi)|$ contracts when increasing $k$, exposing the explicit role of action repetition. The argument works from the block-diagonal factorization \eqref{eq:til_ppi_comp} and requires diagonalizability of each block.

\begin{assumption}[Diagonalizability of action blocks]
\label{assumption:diag_blocks}
For all actions $a \in \mathcal{A}$ the corresponding state transition matrix $P_a$ is
diagonalizable over $\mathbb{C}$, i.e.\ $P_a = U_a \Lambda_a U_a^{-1}$.
\end{assumption}

As with Assumption~\ref{assumption:diag_joint}, this is generic: matrices with distinct eigenvalues are dense in $M_n(\mathbb{R})$ \citep[Thm.\ 2.4.7.1]{horn2013matrix}. Note that diagonalizing the per-action blocks $P_a$ is a stronger requirement than diagonalizing the joint matrix $\widetilde P^\pi$, since the proof relies on the block-diagonal eigendecomposition of $\mathbf{P}_\mathcal{A}$ being matched with $\widetilde\pi$.

\EigContractionLem*

\begin{proof}
We work from the factorization \eqref{eq:til_ppi_comp}.
Per Assumption~\ref{assumption:diag_blocks}, each transition matrix $P_{a}$ is diagonalizable as $P_{a} = U_{a} \Lambda_{a} U_{a}^{-1}$, so $\mathbf{P}_{\mathcal{A}} = U \Lambda_{\mathcal{A}} U^{-1}$ with the block-diagonal eigenvector and eigenvalue matrices
\begin{equation*}
U = \mathrm{diag}(U_{a_1}, \dots, U_{a_{|\mathcal{A}|}}), \quad \Lambda_\mathcal{A} = \mathrm{diag}(\Lambda_{a_1}, \dotsc, \Lambda_{a_{|\mathcal{A}|}}).
\end{equation*}
Substituting into \eqref{eq:til_ppi_comp} and raising to the $k$-th power gives
\begin{equation}
\widetilde{P}^\pi = K U \Lambda_{\mathcal{A}}^k U^{-1} E \tilde{\pi}.
\end{equation}
To extract the $(d+1)$-th eigenvalue, we partition the 1-step spectrum $\Lambda_{\mathcal{A}}$ into the $d$ largest eigenvalues ($\Lambda_{\mathrm{slow}}$) and the remaining fast-mixing eigenvalues ($\Lambda_{\mathrm{fast}}$):
\begin{equation}
\Lambda_{\mathcal{A}} = \begin{pmatrix} \Lambda_{\mathrm{slow}} & 0 \\ 0 & \Lambda_{\mathrm{fast}} \end{pmatrix}.
\end{equation}
By definition, the largest absolute value in $\Lambda_{\mathrm{fast}}$ is exactly $|\lambda_{d+1}(\mathbf{P}_{\mathcal{A}})|$. We now decompose the full system into a rank-$d$ matrix ($M_d$) and a fast-mixing error matrix ($E_k$):
\begin{equation}
\widetilde{P}^\pi = \underbrace{K U \begin{pmatrix} \Lambda_{\mathrm{slow}}^k & 0 \\ 0 & 0 \end{pmatrix} U^{-1} E \tilde{\pi}}_{\eqqcolon M_d} + \underbrace{K U \begin{pmatrix} 0 & 0 \\ 0 & \Lambda_{\mathrm{fast}}^k \end{pmatrix} U^{-1} E \tilde{\pi}}_{\eqqcolon E_k}.
\end{equation}

Let $S$ be the eigenvector matrix of $\widetilde{P}^\pi$ (Assumption~\ref{assumption:diag_joint} is invoked here only to make $\kappa(S)$ well-defined; the contraction itself relies only on Assumption~\ref{assumption:diag_blocks}).
To bound the $(d+1)$-th eigenvalue of the system, we rely on global eigenvalue matching bounds for diagonalizable matrices. Let $N = |\mathcal{S}||\mathcal{A}|$ be the dimension of the space. Theorem 3.3 in \cite{stewart1990matrix} establishes the existence of an optimal permutation $\tau^\star$ matching the spectra of $\widetilde{P}^\pi$ and $M_d$ that minimizes the maximum deviation between paired eigenvalues:
\begin{equation}
\min_\tau \max_i |\lambda_i(\widetilde{P}^\pi) - \lambda_{\tau(i)}(M_d)| \leq (2N - 1) \kappa(S)\|E_k\|_2.
\end{equation}
Because $M_d$ has rank at most $d$, it possesses at least $N-d$ zero eigenvalues, so $M_d$ has at most $d$ nonzero eigenvalues.
By the pigeonhole principle applied to the top $d{+}1$ eigenvalues of $\widetilde{P}^\pi$ (those with modulus nonsmaller than $|\lambda_{d+1}(\widetilde{P}^\pi)|$), at least one index $i^\star \in \{1,\dots,d{+}1\}$ must satisfy $\lambda_{\tau^\star(i^\star)}(M_d) = 0$. Combining the matching bound at $i^\star$ with the modulus ordering:
\begin{equation}
|\lambda_{d+1}(\widetilde{P}^\pi)| \;\leq\; |\lambda_{i^\star}(\widetilde{P}^\pi)| \;=\; |\lambda_{i^\star}(\widetilde{P}^\pi) - \lambda_{\tau^\star(i^\star)}(M_d)| \;\leq\; (2N - 1) \kappa(S)\|E_k\|_2,
\end{equation}
where the first inequality holds because $i^\star \leq d{+}1$ and the eigenvalues are ordered by nonincreasing modulus.

Using the submultiplicativity of the spectral norm, we bound $\|E_k\|_2$:
\begin{align}
\|E_k\|_2 &\leq \|K\|_2 \cdot \|U\|_2 \cdot \|\Lambda_{\mathrm{fast}}^k\|_2 \cdot \|U^{-1}\|_2 \cdot \|E \tilde{\pi}\|_2 \\
&\leq 1 \cdot \kappa(U) \cdot |\lambda_{d+1}(\mathbf{P}_{\mathcal{A}})|^k \cdot \sqrt{|\mathcal{A}|},
\end{align}
where we use the fact that $K$ is a permutation matrix ($\|K\|_2 = 1$), $\kappa(U) = \|U\|_2 \|U^{-1}\|_2$, and we bound $\|E \tilde{\pi}\|_2 \leq \|E\|_2 \leq \sqrt{|\mathcal{A}|}$. Substituting this back yields the explicit bound on the $(d+1)$-th eigenvalue of the system:
\begin{equation}
|\lambda_{d+1}(\widetilde{P}^\pi)| \leq \underbrace{(2|\mathcal{S}||\mathcal{A}| - 1) \kappa(S) \kappa(U) \sqrt{|\mathcal{A}|}}_{\eqqcolon C_{\mathrm{rep}}} \; \cdot \; |\lambda_{d+1}(\mathbf{P}_{\mathcal{A}})|^k,
\end{equation}
which is the claim of the lemma.
\end{proof}

A structural artifact of the block-diagonal factorization \eqref{eq:til_ppi_comp} is that each block $P_{a}$ is row-stochastic, so $\mathbf{P}_\mathcal{A}$ has at least $|\mathcal{A}|$ unit eigenvalues. Consequently, $|\lambda_{d+1}(\mathbf{P}_\mathcal{A})| < 1$ requires $d \geq |\mathcal{A}|$; otherwise Lemma~\ref{lem:eig_contraction} is vacuous because $|\lambda_{d+1}(\mathbf{P}_\mathcal{A})|^k = 1$ for all $k$.
This constraint is specific to the block diagonal proof strategy used here.

Combining this with Lemma~\ref{lem:the_lemma} via direct substitution yields the explicit form of the spectral error term as
\begin{equation*}
    \frac{C_{\mathrm{SF}}}{1 - \gamma^k C_{\mathrm{rep}} |\lambda_{d+1}(\mathbf{P}_\mathcal{A})|^k},
\end{equation*}
which is meaningful whenever $C_{\mathrm{rep}} |\lambda_{d+1}(\mathbf{P}_\mathcal{A})|^k < \gamma^{-k}$.
Under stronger structural assumptions such as orthogonality of the action block matrices $P_a$ the same proof strategy would potentially tighten $C_\mathrm{rep}$ to $\sqrt{|\mathcal{A}|}$.

\newpage
\section{Spectral Metrics}
\label{sec:spectral_metrics}

We define two spectral metrics to quantify the effective rank of the \ac{SR}.

\paragraph{Stable Rank.}
The stable rank captures the concentration of spectral energy relative to the dominant singular direction. 
For a matrix $M$, it is defined as:
\[
\mathrm{SRank}(M) = \frac{\|M\|_F^2}{\|M\|_2^2}
= \frac{\sum_i \sigma_i^2}{\sigma_1^2},
\]
where $\{\sigma_i\}$ are the singular values of $M$. 
Lower values indicate stronger concentration in leading components and thus greater low-rank structure.

\paragraph{Normalized Spectral Entropy.}
Normalized spectral entropy (NSE) measures how evenly spectral energy is distributed:
\[
\mathrm{NSE}(M) = \frac{-\sum_i p_i \log p_i}{\log(\beta)}, 
\quad 
p_i = \frac{\sigma_i^2}{\sum_j \sigma_j^2},
\]
where $\beta$ denotes the number of singular values. 
This normalization ensures $\mathrm{NSE}(M) \in [0,1]$. 
Higher values correspond to a more diffuse spectrum, while lower values indicate concentration in a few modes.

\paragraph{Discrete Setting.}
In discrete environments, we compute both metrics directly on the exact \ac{SR} matrix $M^{\pi}$ by performing singular value decomposition (SVD) to obtain $\{\sigma_i\}$.

\paragraph{Continuous Setting.}
In continuous domains, where the exact \ac{SR} is unavailable, we evaluate the metrics on an empirical approximation 
$\hat{M}^{\pi} = F B^\top$, constructed from transitions collected under the same exploration protocol used during training. 
In contrast to training—where a batch of latent embeddings is sampled—we use a single randomly sampled latent embedding shared across all transitions, yielding an estimate of $\hat{M}^{\pi}$ for a fixed (random) goal.

To mitigate scale drift arising from variations in embedding norms, we apply a row-wise softmax normalization such that each row sums to $\frac{1}{1-\gamma}$. 
We then compute the singular values of the normalized matrix via SVD and evaluate the spectral metrics as in the discrete case.
All reported results correspond to averages of the metrics over all trained models using random seeds (see Figure~\ref{fig:spectral_analysis}).

\newpage

\section{Empirical Analysis of the Spectral Bound}
\label{sec:tabular_bound}

\begin{figure}[h]
    \centering
    \includegraphics[width=0.5\textwidth]{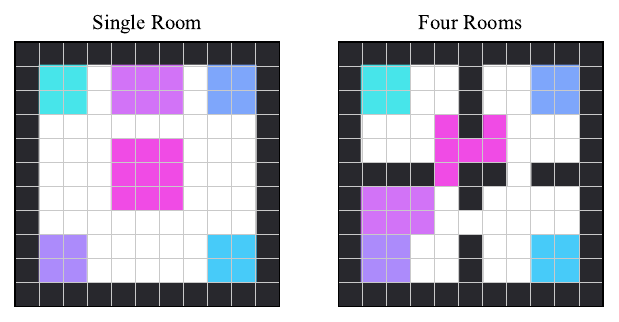}
    \caption{\textbf{Discrete Settings.} Single-Room (left) is an open area; Four-Rooms (right) is divided by an internal wall cross with door openings. Black cells are walls, white cells are free states, and colored cells mark the six goal regions.}
    \label{fig:tabular_layouts}
\end{figure}

{To examine the terms of Lemma~\ref{lem:the_lemma}, we use discrete \emph{Single-Room} and \emph{Four-Rooms} gridworlds on an $11\times 11$ grid. Single-Room contains an open $9\times 9$ area ($81$ states), while Four-Rooms adds internal walls and doorways ($68$ states). Both environments have nine stochastic actions: movement in the eight compass directions and remaining in place. The intended action is taken with probability $0.8$, with the remaining probability distributed uniformly over the other actions.

Each environment contains six continuing binary goal-region rewards with centers chosen by farthest-point sampling. A state receives reward $1$ within Euclidean distance $1.75$ of a goal center and $0$ otherwise. Each goal is a separate task; repeated actions accumulate reward at every primitive step, including intermediate visits to the goal region. Four-Rooms matches the topology of the continuous benchmark used in the main text.

We follow the experimental convention of Section~\ref{sec:experiments}: $\gamma$ is the nominal discount of the action-repeat MDP, while $\gamma_{\text{eff}}=\gamma^{1/k}$ is the original-environment discount. We hold the nominal $\gamma$ fixed across repeat factors.

For each environment, we fit a shared rank-$d$ \ac{FB} representation to the six goal-induced Boltzmann policies ($\tau=0.3$, under-relaxation $0.5$). We use $\pi_Q(a\mid s)\propto\exp(Q(s,a)/\tau)$ and evaluate its Q-function through $Q=r+\gamma P\mathbb{E}_{a\sim\pi_Q}Q$. Q-functions and \ac{SR} operators are computed from the known dynamics, without sampled trajectories. We alternate between fitting the shared representation to the current policies' successor representations and recomputing the induced policies. When a fit does not converge, we report the iterate with the smallest policy change.

\paragraph{Realization error.}
Table~\ref{tab:tabular_realization} reports median excess realization error and the spectral floor $\sigma_{d+1}(\widetilde M^\pi)$ over six rewards, nominal discounts $\gamma\in\{0.95,0.99\}$, and repeat factors $k\in\{1,2,3,5,9\}$. Both terms generally decrease with $d$. The spectral floor is the best rank-$d$ approximation error; excess realization error measures the additional cost of fitting a shared representation.

\begin{table}[!htbp]
    \centering
    \caption{Median excess realization error and spectral rank floor on the $11\times11$ goal-task grids. These are \ac{SR} matrix errors, not Q-value errors.}
    \label{tab:tabular_realization}
    \begin{tabular}{llccc}
        \toprule
        Environment & Quantity & $d=16$ & $d=32$ & $d=64$ \\
        \midrule
        Single-Room & $\widetilde\epsilon_{\mathrm{real}}$ & $2.18$ & $1.51$ & $1.19$ \\
                    & $\sigma_{d+1}(\widetilde M^\pi)$ & $2.50$ & $1.99$ & $1.41$ \\
        \addlinespace
        Four-Rooms  & $\widetilde\epsilon_{\mathrm{real}}$ & $2.20$ & $1.58$ & $1.57$ \\
                    & $\sigma_{d+1}(\widetilde M^\pi)$ & $2.76$ & $2.05$ & $1.18$ \\
        \bottomrule
    \end{tabular}
\end{table}

\paragraph{Action-repeat error.}
Let $Q_{1,\tau}$ and $Q_{k,\tau}$ denote the exact Boltzmann Q-functions in the original and $k$-repeat MDPs, respectively. Table~\ref{tab:tabular_repeat} reports the median $\|Q_{1,\tau}-Q_{k,\tau}\|_\infty/\|Q_{1,\tau}\|_\infty$ across six goals. Unlike Definition~\ref{def:repeat_error}, this measures Boltzmann-policy rather than hard-optimal Q-values.

\begin{table}[!htbp]
    \centering
    \caption{Action-repeat Q-value error for Boltzmann policies, normalized by $\|Q_{1,\tau}\|_\infty$, at nominal $\gamma=0.99$. Median over six goals.}
    \label{tab:tabular_repeat}
    \begin{tabular}{lcccc}
        \toprule
        Environment & $k=2$ & $k=3$ & $k=5$ & $k=9$ \\
        \midrule
        Single-Room & $2.49\%$ & $2.48\%$ & $1.78\%$ & $1.06\%$ \\
        Four-Rooms  & $2.18\%$ & $2.30\%$ & $1.94\%$ & $2.72\%$ \\
        \bottomrule
    \end{tabular}
\end{table}

\paragraph{Spectral \ac{SR} estimate.}
Table~\ref{tab:tabular_spectral} compares the exact truncation error $\sigma_{d+1}(\widetilde M^\pi)$ with the spectral estimate $C_{\mathrm{SF}}/[1-\gamma|\lambda_{d+1}(\widetilde P^\pi)|]$ for the same Boltzmann-policy goals (with $\gamma$ the repeated-MDP discount). $C_{\mathrm{SF}}$ is calibrated on these same goals and repeats at fixed layout, discount, and rank, so the ratios describe an in-sample diagnostic, not a test of the theoretical constant or of Lemma~\ref{lem:eig_contraction}.

\begin{table}[!htbp]
    \centering
    \caption{Calibrated spectral estimate divided by exact $\sigma_{d+1}(\widetilde M^\pi)$ for Boltzmann policies. Median and maximum over six goals, $\gamma\in\{0.95,0.99\}$, and $d\in\{16,32,64\}$; calibration uses the same cases.}
    \label{tab:tabular_spectral}
    \begin{tabular}{llccccc}
        \toprule
        Environment & Statistic & $k=1$ & $k=2$ & $k=3$ & $k=5$ & $k=9$ \\
        \midrule
        Single-Room & Median  & $1.28$ & $1.13$ & $1.13$ & $1.31$ & $1.59$ \\
                    & Maximum & $1.58$ & $1.21$ & $1.21$ & $1.66$ & $2.03$ \\
        \addlinespace
        Four-Rooms  & Median  & $1.19$ & $1.05$ & $1.06$ & $1.22$ & $1.40$ \\
                    & Maximum & $1.44$ & $1.13$ & $1.19$ & $1.30$ & $1.55$ \\
        \bottomrule
    \end{tabular}
\end{table}

\paragraph{Full bound.}
We do not evaluate the full optimality-gap bound here: it concerns hard-greedy policies, whereas these fits use Boltzmann policies and do not always converge.

\newpage
\section{Extra Plots and Ablations}
\label{sec:extra_plots}

\subsection{Overall effects of $k$, $\gamma$, and $d$}
\label{subsec:overall_parallel}
Figure~\ref{fig:supplement-parallel_plot} provides an overview of the effect of the three main hyperparameters of \ac{FB} on the final episodic return of the Four-Rooms continuous environment. Two values are of particular importance: action-repetition ($k$=1) and nominal discount factor ($\gamma$=0.999). In both cases, the performance suffers significantly regardless of the values of other hyperparameters. In the case of $k$=1 or no temporal abstraction, FB networks find it challenging to learn a good representation due to the presence of unpredictable high-frequency dynamical modes. In the case of a high discount factor, $\gamma$=0.999, a good representation cannot be achieved as the representation rank approaches singularity. The learning is less sensitive overall to the value of the embedding dimension $d$. 

\begin{figure}[h]
    \begin{center}
        \includegraphics[scale=0.5]{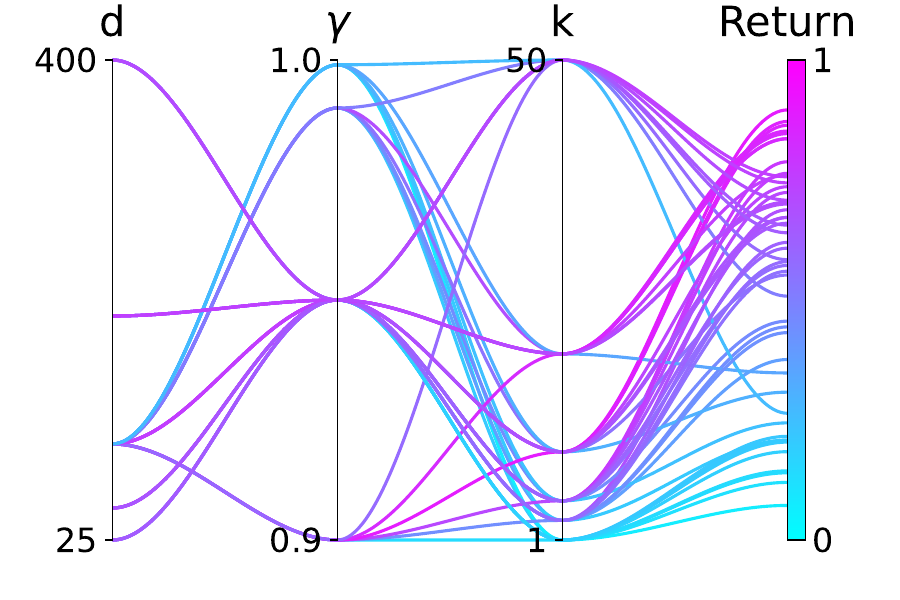}
    \end{center}
    \caption{
    \textbf{Relationship between the main hyperparameters and episodic return.} This plot gives an overview of the effect of different combinations of embedding dimension ($d$), discount factor ($\gamma$), and temporal abstraction $k$ over all experiments. Noticeably, $k=1$ or $\gamma=0.999$ leads to poor performance in most combinations.
    }
    \label{fig:supplement-parallel_plot}
\end{figure}

\newpage

\subsection{Training plots: Ablation of $k$, $\gamma$, and $d$}
\label{subsec:traiing_ablation}
Figure~\ref{fig:supplement-training} shows the performance of different combinations of the main hyperparameters during training on the Four-Rooms environment, with a focus on the effect of introducing temporal abstraction.  Figures~\ref{fig:training-vary-z} and \ref{fig:training-vary-gamma} highlight that without temporal abstraction ($k$=1), varying the embedding dimension $d$ or the discount factor $\gamma$ yields no significant improvement. Figure~\ref{fig:training-vary-k}, on the other hand, shows that even a small level of temporal abstraction ($k$=3) can lead to a significant boost in performance. The figure also shows the limitation of the temporal abstraction, where a large temporal abstraction ($k$=50) can start to negatively impact performance by oversimplifying the SR representation and removing dynamical modes useful for the navigation task.

\begin{figure}[h]
    \centering
    \begin{subfigure}[b]{0.50\textwidth}
        \centering
        \includegraphics[width=0.99\textwidth]{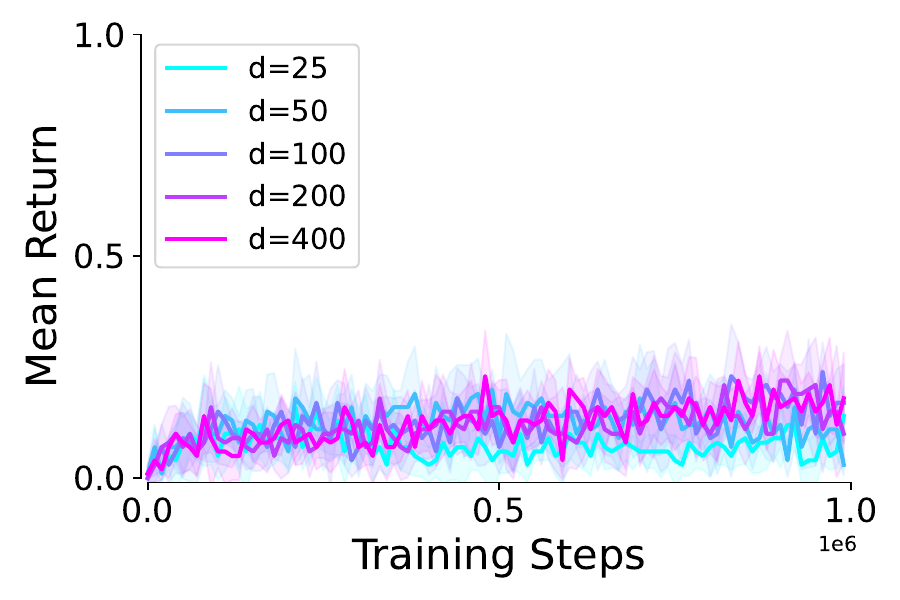}
        \caption{$\gamma$=0.95, $k$=1}
        \label{fig:training-vary-z}
    \end{subfigure}
    \begin{subfigure}[b]{0.50\textwidth}
        \centering
        \includegraphics[width=0.99\textwidth]{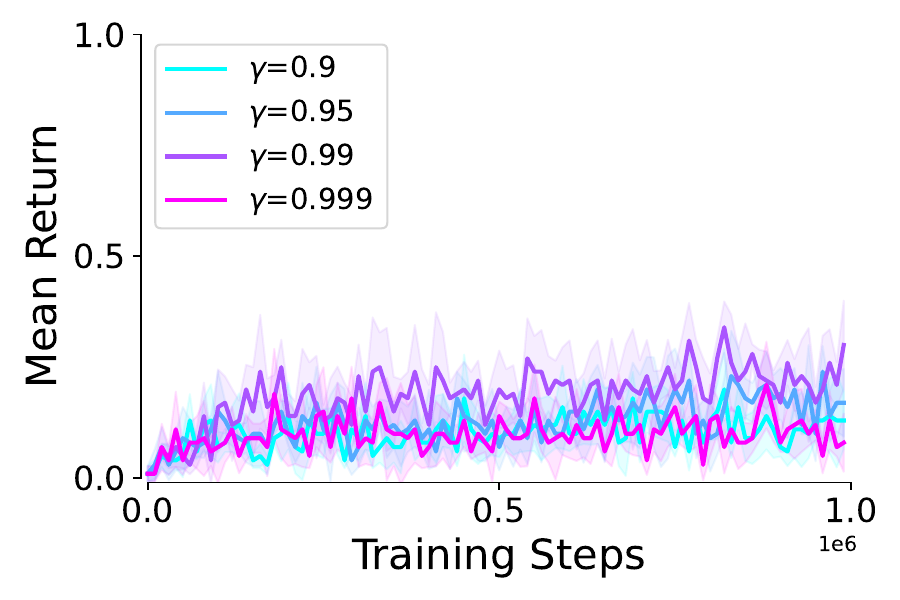}
        \caption{$d$=100, $k$=1}
        \label{fig:training-vary-gamma}
    \end{subfigure}
    \begin{subfigure}[b]{0.50\textwidth}
        \centering
        \includegraphics[width=0.99\textwidth]{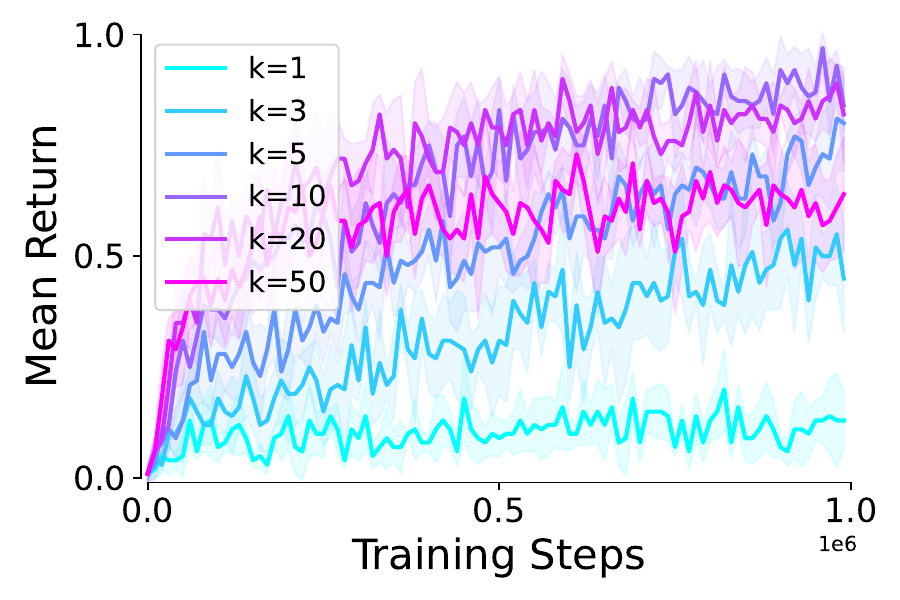}
        \caption{$d$=100, $\gamma$=0.95}
        \label{fig:training-vary-k}
    \end{subfigure}
    \caption{\textbf{Ablation: Training plots.} Increasing embedding dimension $d$ or discount factor $\gamma$ without increasing the temporal abstraction does not yield a meaningful increase in performance.}
    \label{fig:supplement-training}
\end{figure}

\newpage

\subsection{SR and its Q-function for discrete and continuous Four-Rooms environment}
\label{subsec:sr_and_q_complete}
Figure~\ref{fig:cover_full} shows a more complete picture of SR and its associated Q function (mean over cardinal action directions).
The \textit{Baseline} shows the SR and Q using no temporal abstraction ($k$=1), and a moderate discount factor ($\gamma$=0.95). 
For the continuous settings, SR is calculated via FB with embedding dimension ($d$=100). 

In discrete settings (top two rows),  a low-rank structure can be achieved in three ways: 1) SVD with a small rank \textit{(rank = 4)}, 2) high discount factor \textit{($\gamma$=0.999)}, or 3) using temporal abstraction via action repetition \textit{($k$=10)}. In the absence of function approximation and bootstrapping, all three paths lead to an overall similar result where a low-rank structure can remove the high-frequency dynamical modes and create shared future topology (rooms, corridors,...) where states with similar reachability are grouped together and have similar values.

In continuous settings (bottom two rows), where SR and its associated Q are learned via FB using function approximation and bootstrapping, the results differ. To enforce a low-rank structure via the FB algorithm, we reduce the embedding dimension from 100 to 25. The figure shows a small smoothing (grouping of states), but this is not nearly close to the effect of enforcing low-rank structure using SVD in the discrete setting. Increasing the discount factor ($\gamma$=0.999) and introducing temporal abstraction via action repetition ($k$=10) show more promise as they both help spread the SR and Q values to the neighboring rooms. However, a closer look at the Q values shows that only temporal abstraction can smoothly distribute the Q values as the states move away from the goal (start marker). The policy based on increased $\gamma$ will be stuck in local maxima, while the policy based on increased $k$ can follow the Q gradients to the goal. 

\begin{figure}[h]
    \begin{center}
        \includegraphics[scale=0.7]{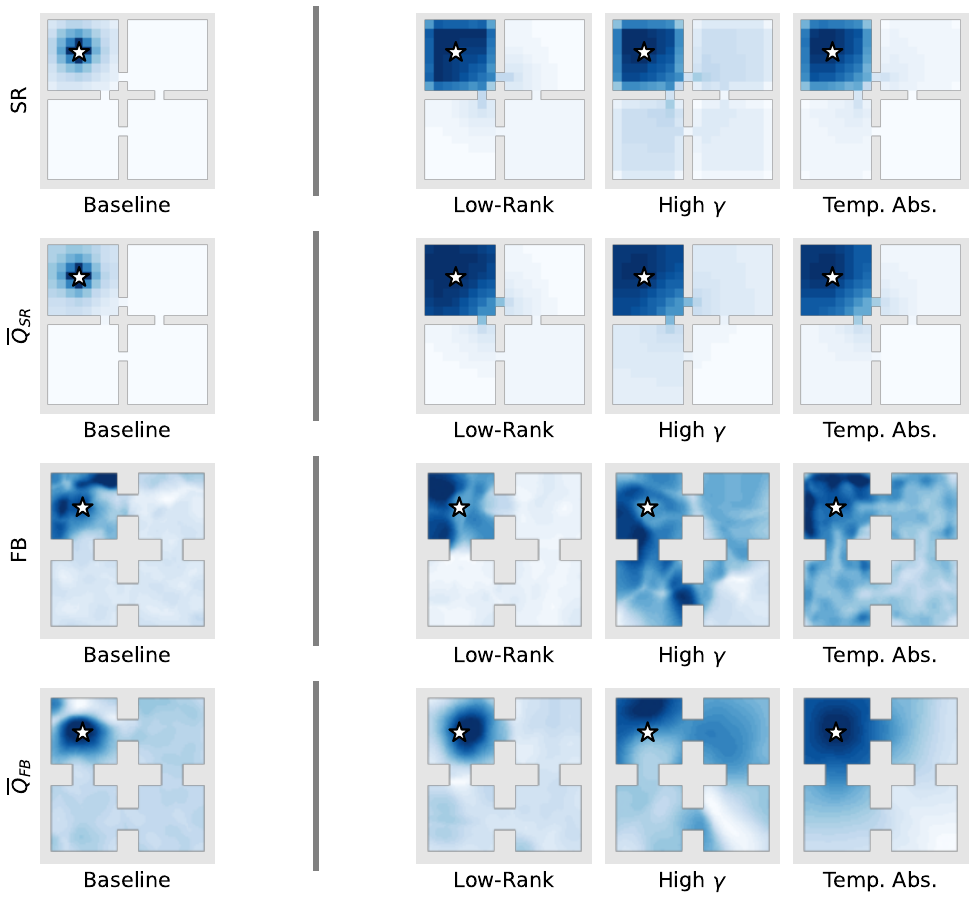}
    \end{center}
    \caption{
    \textbf{Successor Representation (SR) and its Q-Function - Discrete and Continuous} This is a more complete picture of Figure~\ref{fig:cover}. The Q-functions are derived from the same SR that is presented here. The \textit{star marker} marks the starting state for SR and the goal state for the Q-function.
    }
    \label{fig:cover_full}
\end{figure}

\newpage
\subsection{Absolute vs Relative Bellman Error}
\label{subsec:abs_vs_rel_bellman}
Figure~\ref{fig:normalized_bellman_ablations} presents the normalized Bellman error corresponding to Figure~\ref{fig:bellman_ablations}, where the residuals are scaled by the magnitude of the $Q$-values. 
We observe that increasing the embedding dimension has a negligible effect on the relative Bellman error, whereas increasing the degree of temporal abstraction consistently reduces it.

This reduction becomes more pronounced as the discount factor $\gamma$ increases, in contrast to the trend observed for the absolute Bellman error. 
Overall, the results reveal a clear divergence between these metrics at large $\gamma$: the relative Bellman error decreases, while the episodic return simultaneously deteriorates (Figure~\ref{fig:k_gamma_ablation_barplot}).

We hypothesize that this discrepancy is driven by the growth of the absolute Bellman error. 
As $\gamma \to 1$, the scale of the successor representation increases proportionally to the effective horizon, $(1 - \gamma)^{-1}$, which artificially attenuates the normalized error. 
However, optimization is governed by the absolute Bellman residual. 
Thus, larger absolute errors at high $\gamma$ lead to increased gradient variance and a weaker contraction effect, resulting in training instability. 

These findings suggest that the absolute Bellman error is a more reliable indicator of policy degradation than its normalized counterpart, as it more faithfully captures the intrinsic difficulty of function approximation in long-horizon regimes.

\begin{figure}[h]
    \centering
        \begin{subfigure}[b]{0.32\textwidth}
            \centering
            \includegraphics[width=0.9\textwidth]{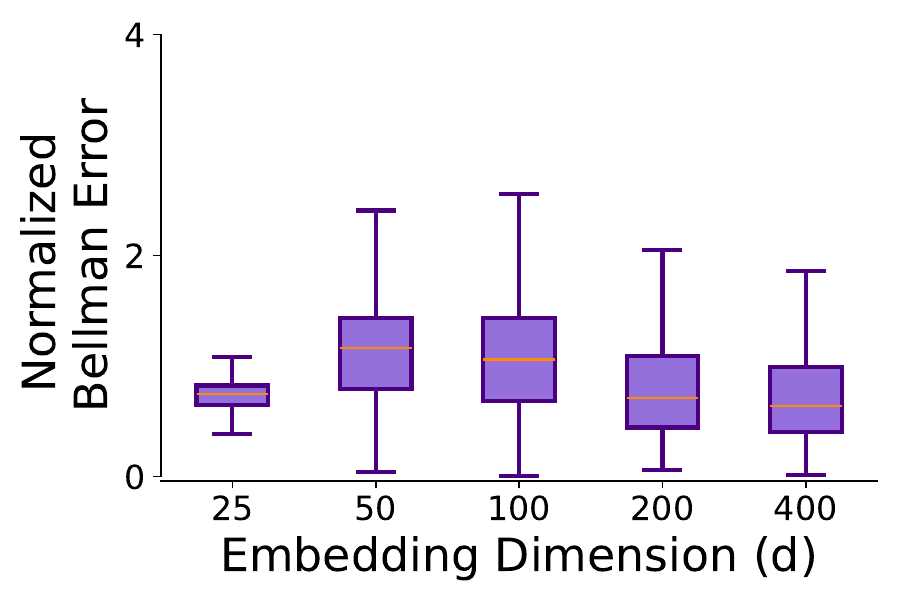}
            \caption{}
            \label{fig:normalized_bellman_vary_z}
        \end{subfigure}
        \begin{subfigure}[b]{0.32\textwidth}
            \centering
            \includegraphics[width=0.9\textwidth]{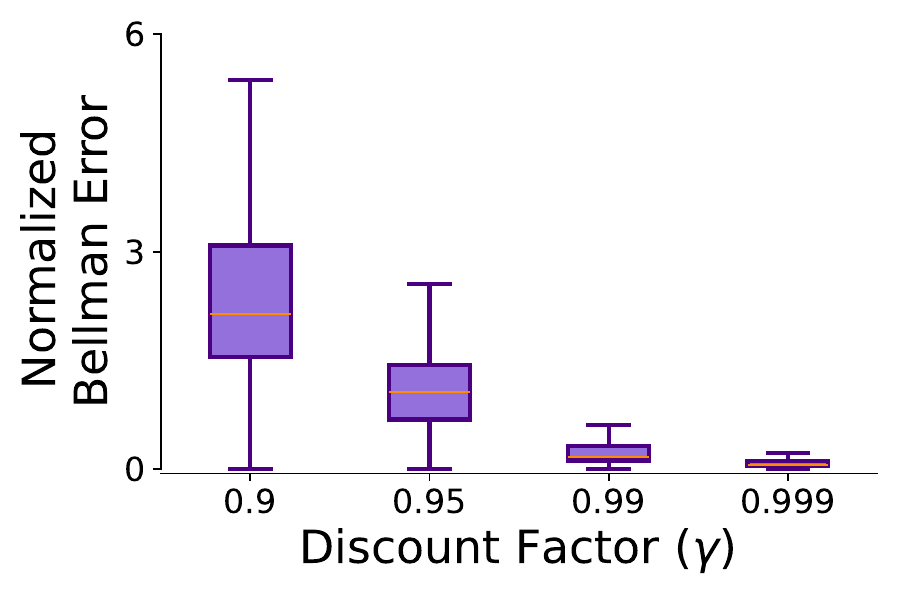}
            \caption{}
            \label{fig:normalized_bellman_vary_gamma}
        \end{subfigure}
        \begin{subfigure}[b]{0.32\textwidth}
            \centering
            \includegraphics[width=0.9\textwidth]{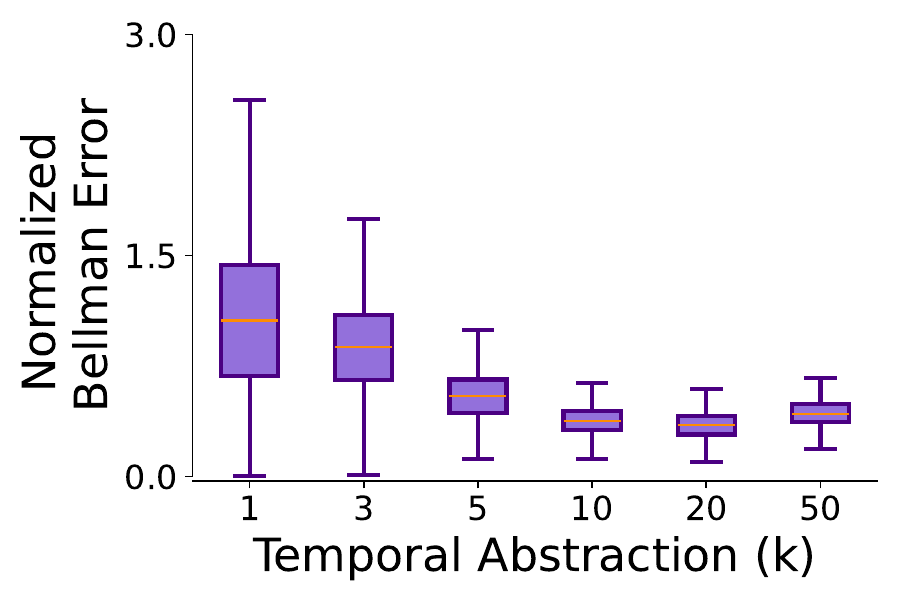}
            \caption{}
            \label{fig:normalized_bellman_vary_k}
        \end{subfigure}
    \caption{\textbf{Ablations: Normalized Bellman error.} Bellman errors are normalized by the Q values. The relative Bellman error decreases sharply as the discount factor is increased  (b). However, this decrease in relative Bellman error does not translate to better performance (episodic return) as discussed in Section~\ref{sec:experiments}.
    }
    \label{fig:normalized_bellman_ablations}
\end{figure}

\newpage
\section{Exploration Coverage}
\label{sec:coverage}
In order to verify that increasing action repetition did not significantly influence the exploration coverage of the state space, we plot the states visited during the training of agents with varying action repetition values for the LargeMaze environment. We use one interaction in the $k$-repeat environment per training step, hence each agent visits one million states during its training.

\begin{figure}[h]
    \centering
        \begin{subfigure}[h]{0.3\textwidth}
            \centering
            \includegraphics[width=0.9\textwidth]{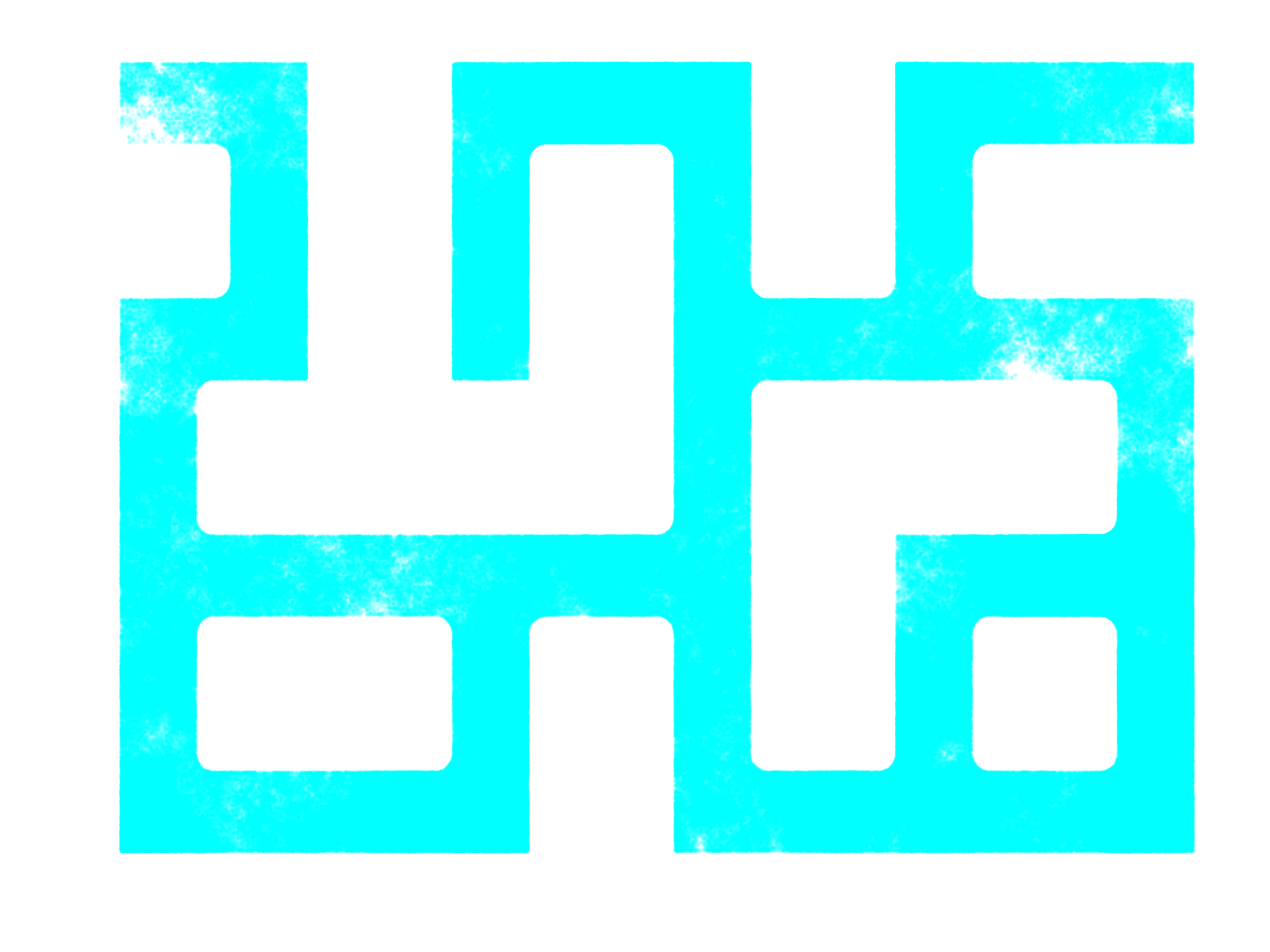}
            \caption{$k=1$}
            \label{}
        \end{subfigure}
        \begin{subfigure}[h]{0.3\textwidth}
            \centering
            \includegraphics[width=0.9\textwidth]{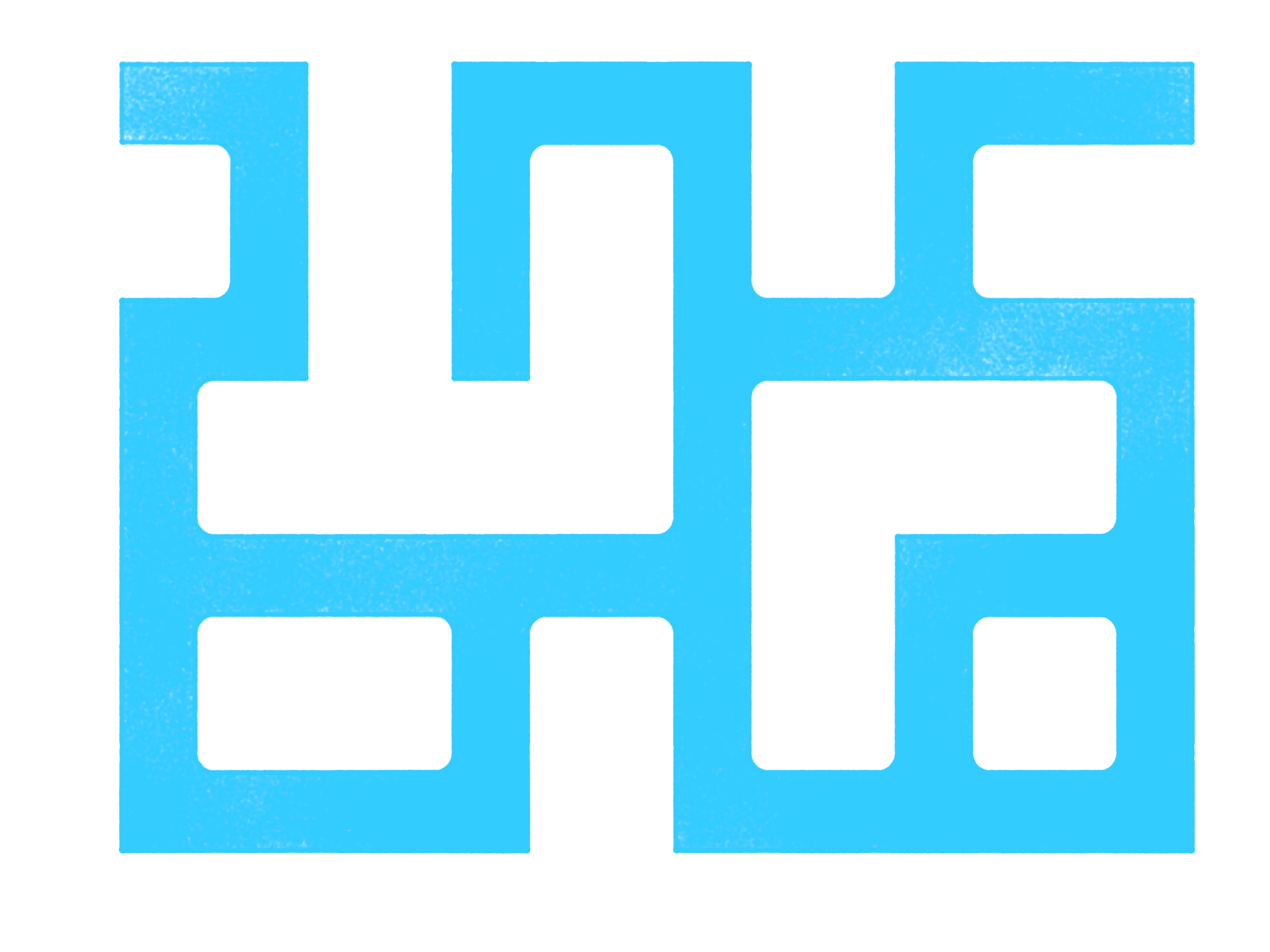}
            \caption{$k=3$}
            \label{}
        \end{subfigure}
        \begin{subfigure}[h]{0.3\textwidth}
            \centering
            \includegraphics[width=0.9\textwidth]{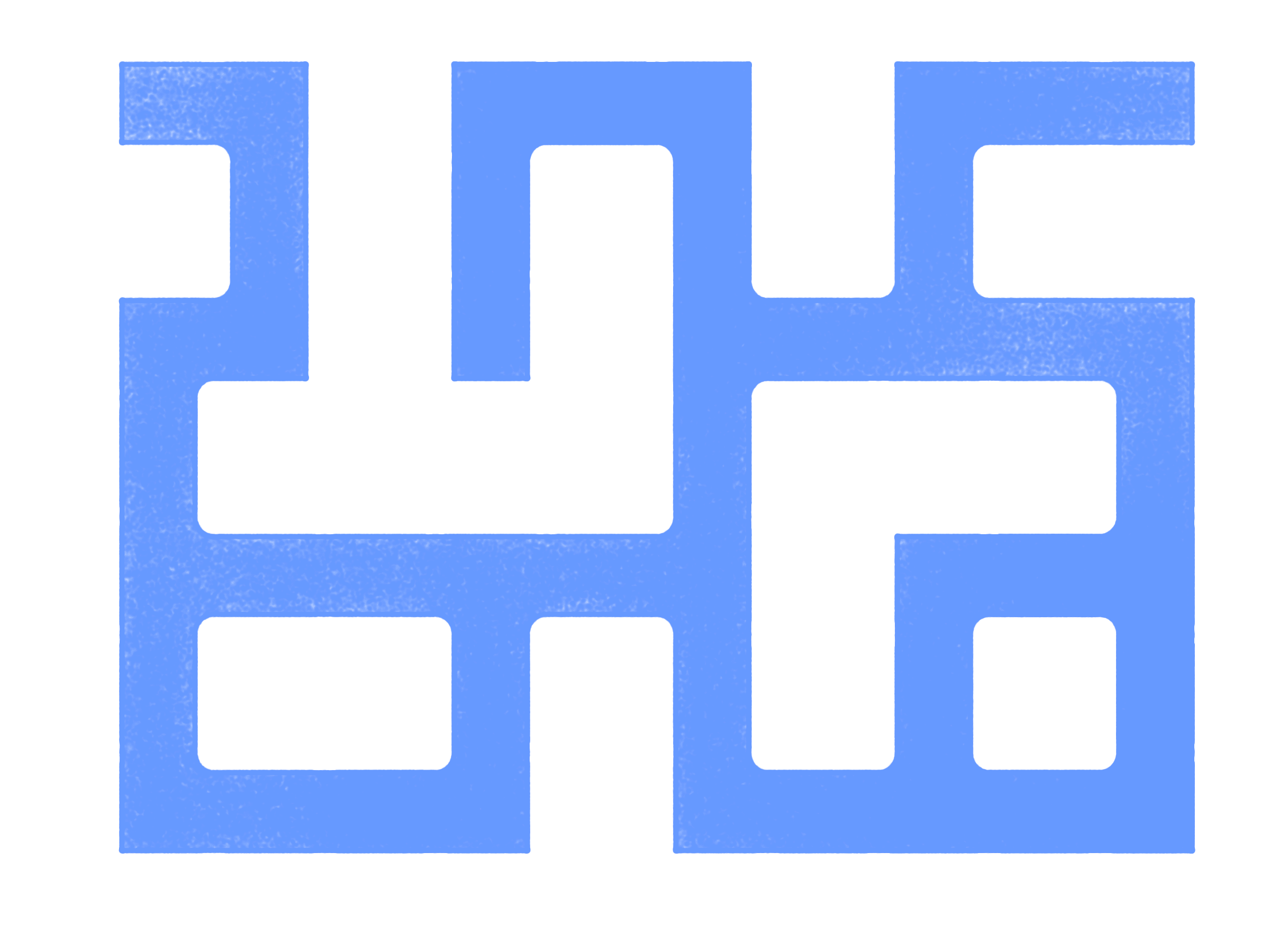}
            \caption{$k=5$}
            \label{}
        \end{subfigure}
        \begin{subfigure}[h]{0.3\textwidth}
            \centering
            \includegraphics[width=0.9\textwidth]{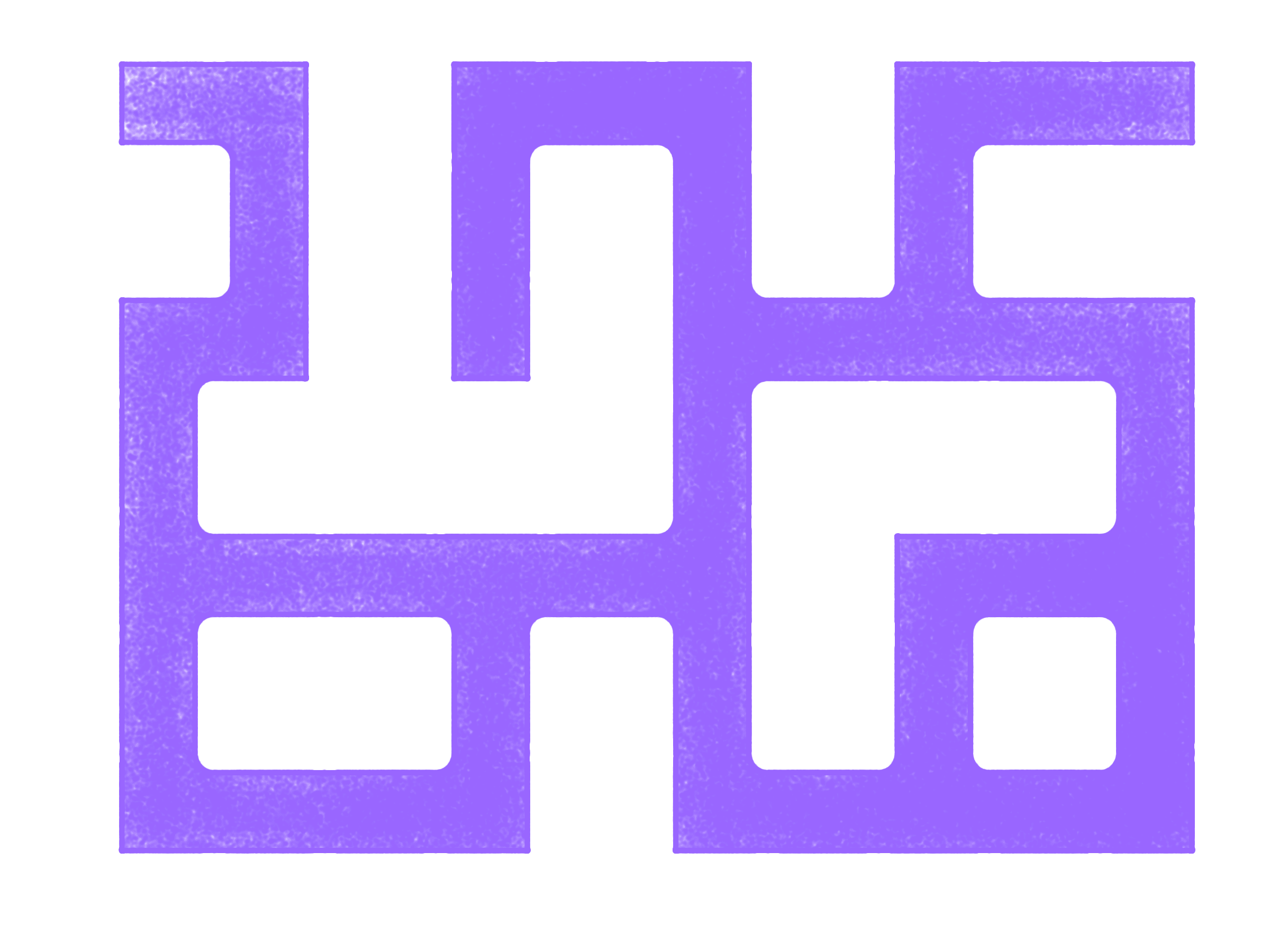}
            \caption{$k=10$}
            \label{}
        \end{subfigure}
        \begin{subfigure}[h]{0.3\textwidth}
            \centering
            \includegraphics[width=0.9\textwidth]{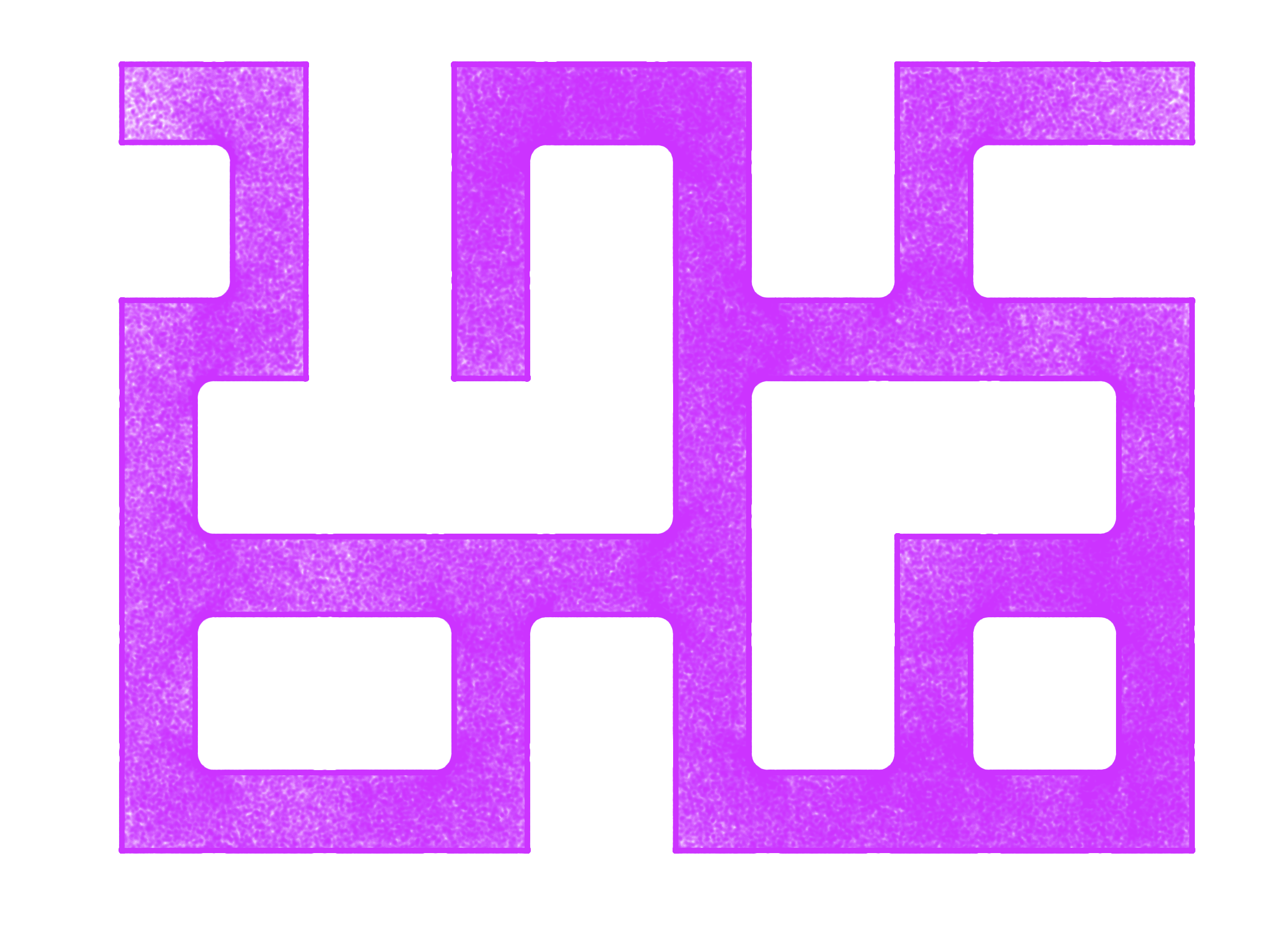}
            \caption{$k=20$}
            \label{}
        \end{subfigure}
        \begin{subfigure}[h]{0.3\textwidth}
            \centering
            \includegraphics[width=0.9\textwidth]{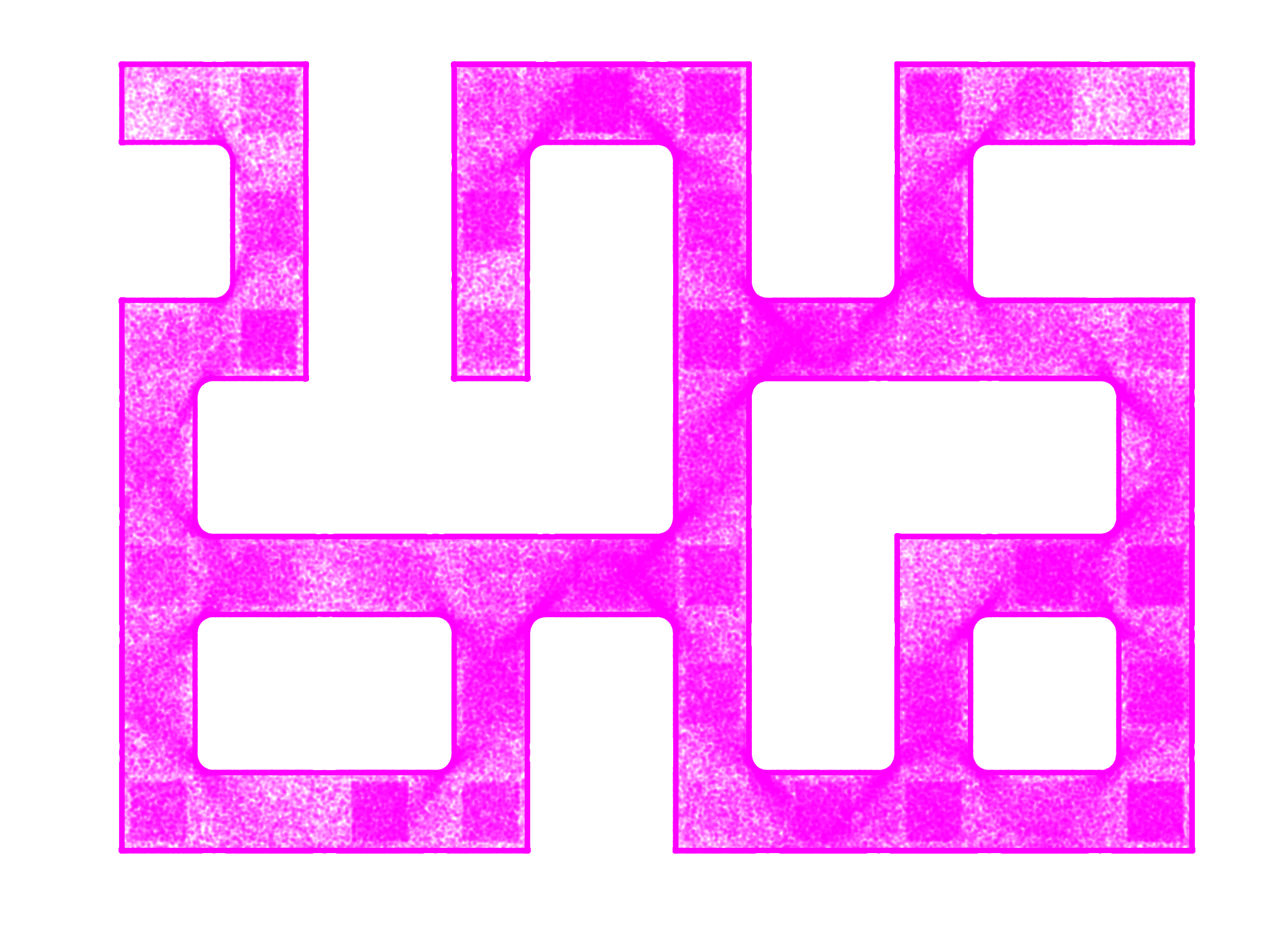}
            \caption{$k=50$}
            \label{}
        \end{subfigure}
    \caption{\textbf{Scatter plot of visited states during training.} Increasing the action repetition does not significantly affect the coverage of the space in the Large-Maze environment. A similar coverage pattern holds in Four-Rooms and Maze environments. (omitted for brevity)}
    \label{fig:coverage}
\end{figure}

\newpage
\section{Manipulation Environment Experiments}
\label{sec:manipulation-exp}
\subsection{Cube Environment}

In addition to the navigation environments, we evaluate temporal abstraction on the Cube manipulation task from the OGBench benchmark suite. 
Unlike the maze environments, which have 2-dimensional state and action spaces, the Cube environment has a 28-dimensional state space and a 5-dimensional action space, making it substantially more challenging. 
To apply the FB algorithm, we represent the goal using the desired $(x,y,z)$ coordinates of the cube.

Figure~\ref{fig:cube-vary-k} shows that, unlike the maze environments, the Cube task benefits from only a small amount of temporal abstraction. 
Nevertheless, the overall trend remains consistent: a small action-repetition factor ($k=3$) already yields a significant improvement in the performance of the FB algorithm.

Prior work on FB has primarily considered tasks in which the agent controls only its own embodiment, without needing to manipulate external objects \citep{touati2021learningone,tirinzoni2025ZeroShotWH,urpi2025epistemicallyguided}. 
The Cube environment is considerably more challenging in this respect, since successful interaction with the cube is unlikely to arise from random online exploration alone. 
We therefore leverage the offline play dataset provided by OGBench. 
When applying action repetition to these offline transitions, we keep the original state and action and replace the next state with the state observed $k$ steps later, implicitly assuming that consecutive actions in the dataset are locally similar.

Figure~\ref{fig:cube-no-data} shows that, without offline play data, the FB algorithm fails to achieve good performance within 1M training steps. 
Extending training to 5M steps partially compensates for the lack of demonstrations in the Cube task, but we do not expect this to scale to more complex manipulation problems that require interacting with multiple objects or executing long sequences of coordinated actions.

\begin{figure}[h]
    \centering
        \begin{subfigure}[h]{0.4\textwidth}
            \centering
            \includegraphics[width=0.9\textwidth]{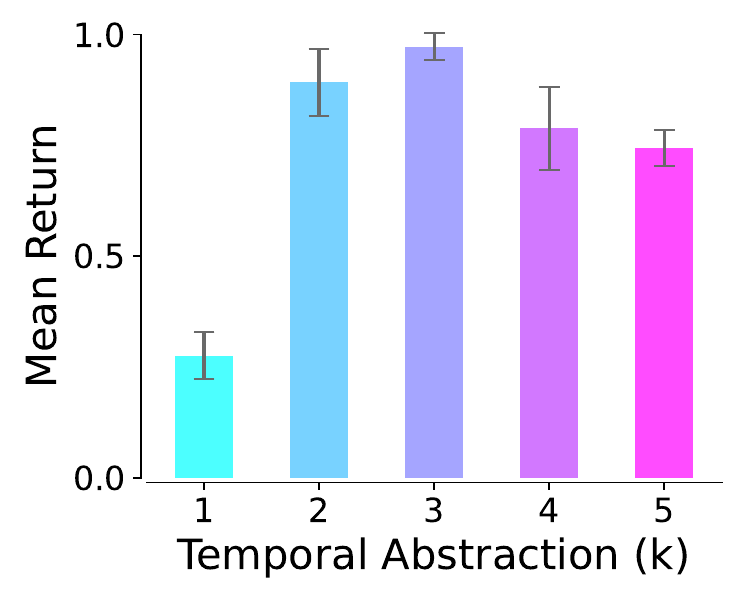}
            \caption{}
            \label{fig:cube-vary-k}
        \end{subfigure}
        \begin{subfigure}[h]{0.4\textwidth}
            \centering
            \includegraphics[width=0.9\textwidth]{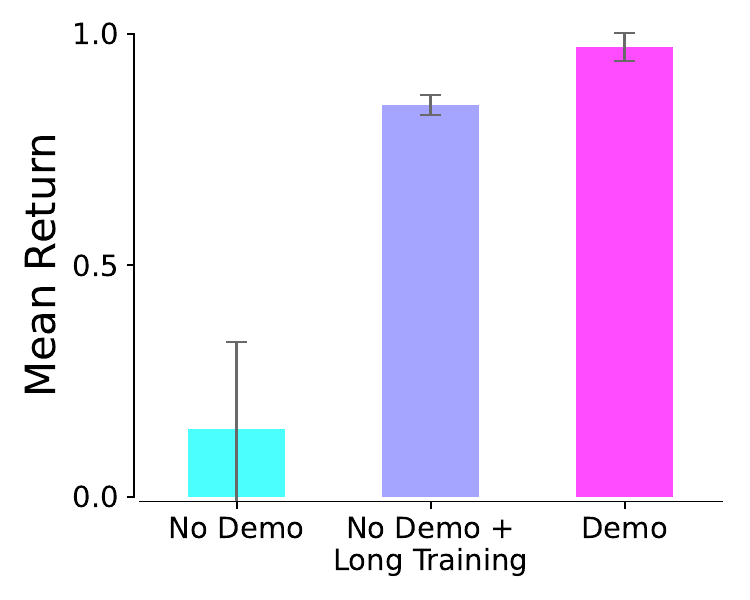}
            \caption{}
            \label{fig:cube-no-data}
        \end{subfigure}
    \caption{\textbf{(a): Temporal abstraction in the Cube environment.} Unlike the maze environments, the Cube task requires only a small amount of temporal abstraction. Nevertheless, applying action repetition ($k=3$) substantially improves the performance of the FB algorithm, consistent with the trend observed in navigation tasks. \textbf{(b): Offline play data versus longer training.} Without offline play data, the FB algorithm requires substantially longer training (5M steps) to achieve competitive performance in the Cube environment.}
    \label{fig:cube-exp}
\end{figure}

\subsection{Scene Environment}
As with the Cube environment, we use the offline play data when learning to solve tasks in the Scene environment. 
Some tasks require precise control (e.g., moving the cube), while others (e.g., opening or closing the window) can be solved efficiently with coarse actions. Goals are defined using states of the two buttons as well as positions of the cube, the drawer and the window.
Figure~\ref{fig:scene-vary-k} shows that, despite this variation, a similar pattern emerges: adding action repetition as a form of temporal abstraction improves overall performance.

Figure~\ref{fig:scene-tasks} examines this more closely across the individual tasks in the Scene environment. 
Each task consists of several sub-tasks (e.g., open drawer $\rightarrow$ place cube in drawer $\rightarrow$ close drawer $\rightarrow$ lock drawer), with the agent receiving a negative reward at each timestep for every incomplete sub-task. 
While action repetition benefits overall performance, the optimal number of repetition steps depends on how fine-grained the controller needs to be. 
\textit{Task 2}, for instance, involves only unlocking and closing the window and drawer, with no cube movement; a coarse controller is sufficient here. 
Other tasks require pressing two buttons or precisely moving the cube, where the benefit of action repetition is less pronounced, since more aggressive repetition reduces the precision needed for good performance.

\begin{figure}[h]
    \centering
        \begin{subfigure}[h]{0.4\textwidth}
            \centering
            \includegraphics[width=0.9\textwidth]{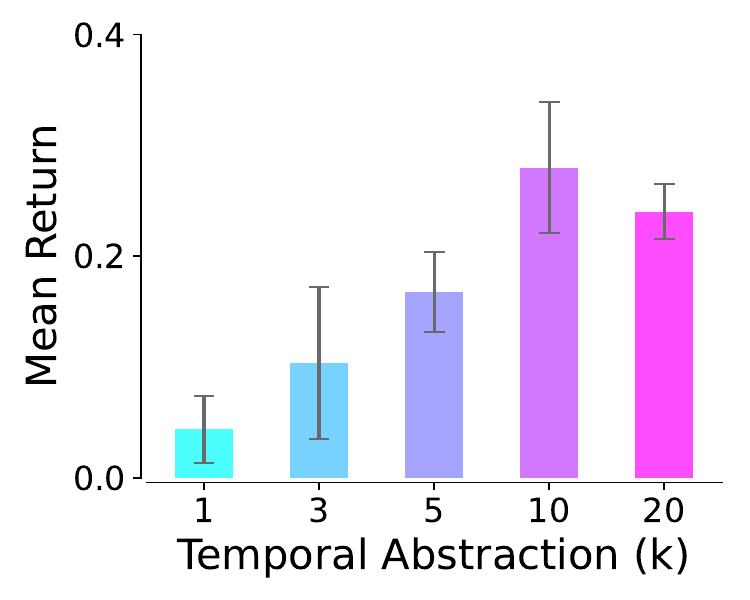}
            \caption{}
            \label{fig:scene-vary-k}
        \end{subfigure}
        \begin{subfigure}[h]{0.45\textwidth}
            \centering
            \includegraphics[width=0.9\textwidth]{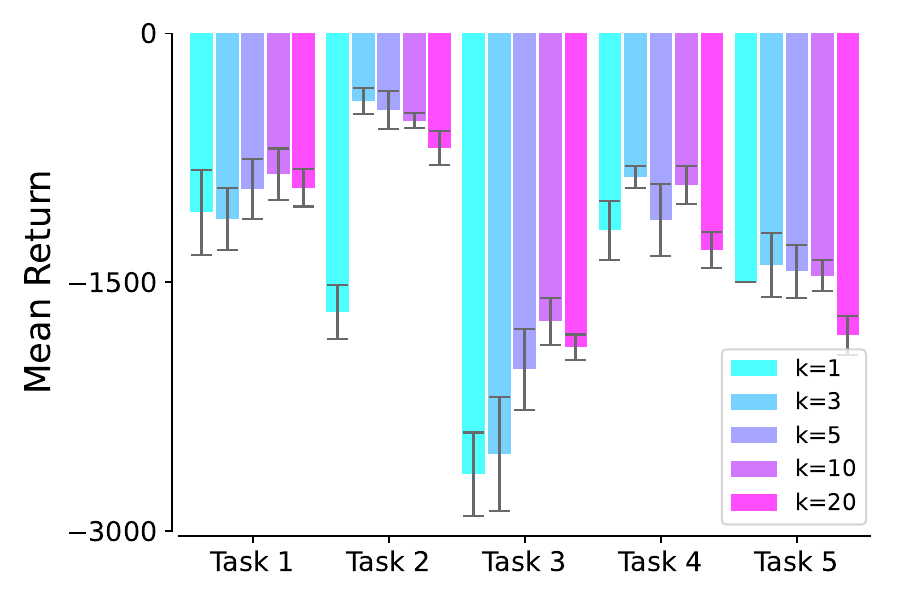}
            \caption{}
            \label{fig:scene-tasks}
        \end{subfigure}
    \caption{\textbf{(a): Temporal abstraction in the Scene environment.} Adding temporal abstraction via action repetition improves overall performance; however, too many repetition steps can hinder the precise control needed by the agent. \textbf{(b): Different tasks of the Scene environment.} Depending on the task, a larger or smaller number of action-repetition steps can be optimal.}
    \label{fig:scene-exp}
\end{figure}

\newpage
\section{Action Chunking as Temporal Abstraction}
\label{sec:action-chunking}
Temporal abstraction can take more sophisticated forms than the one studied in the main paper.
Here, we compare our basic instantiation — action repetition — against action chunking \citep{li2025reinforcement}, a more expressive alternative. 
Unlike action repetition, action chunking expands the action-space dimensionality, which changes the input dimensions of both the actor and forward networks. 
In exchange, the actor can propose a sequence of distinct actions rather than repeating a single one, potentially reducing the coarseness of control that action repetition imposes. 
Figure~\ref{fig:repeat-vs-chunk} compares FB agent return on Large-Maze and Cube under both mechanisms. 
Both offer a significant boost over the baseline ($k=1$), with broadly comparable gains. 
Notably, action chunking sustains its performance advantage at larger k, whereas action repetition's benefit tends to decay — suggesting that the added flexibility of chunking may alleviate the performance drop typically seen with more aggressive repetition. 

\begin{figure}[h]
    \centering
        \begin{subfigure}[h]{0.44\textwidth}
            \centering
            \includegraphics[width=0.9\textwidth]{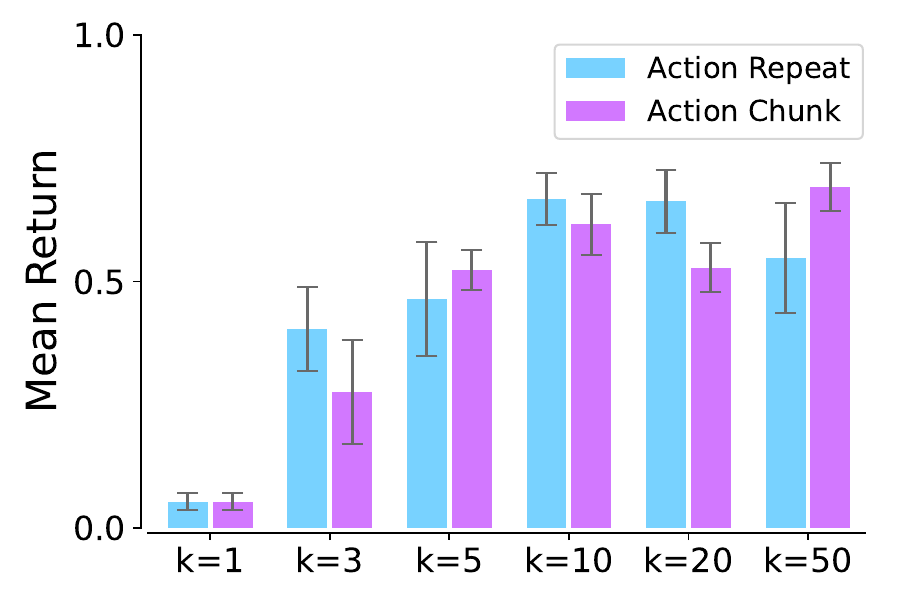}
            \caption{}
            \label{fig:largemaze-repeat-chunk}
        \end{subfigure}
        \begin{subfigure}[h]{0.45\textwidth}
            \centering
            \includegraphics[width=0.9\textwidth]{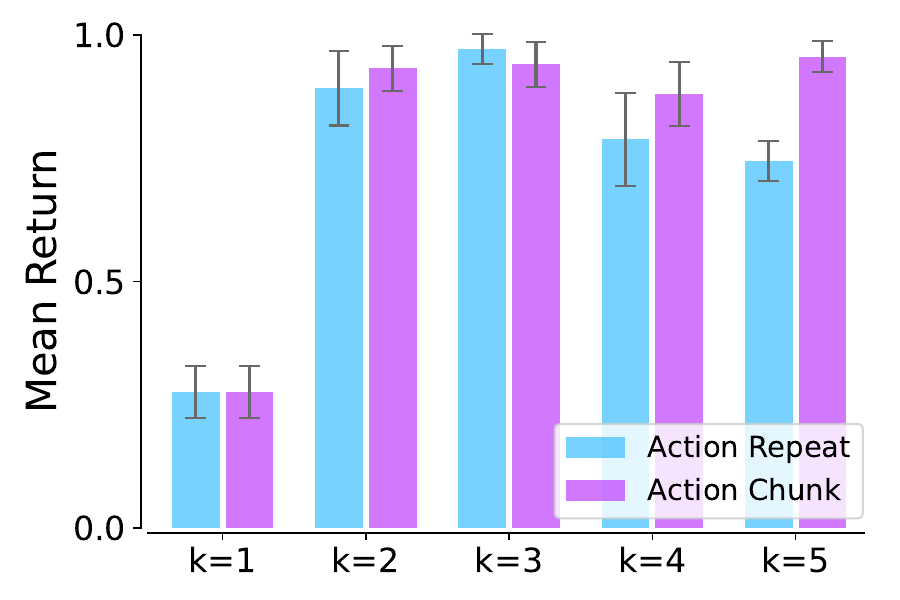}
            \caption{}
            \label{fig:cube-repeat-chunk}
        \end{subfigure}
    \caption{\textbf{Action Repetition vs Action Chunking.} Action repetition or action chunking offers a significant boost in return in Large-Maze (a) and Cube (b) environments. }
    \label{fig:repeat-vs-chunk}
\end{figure}

\newpage
\section{Compute Resources}
\label{sec:compute}
Each experiment (1M training steps) was conducted using a single GPU (NVIDIA GeForce RTX 2080 Ti), taking an average of 12 hours per experiment when training with state observations and an average of 18 hours when using image observations. Each experiment used 1.5GB of RAM when trained using state observations and 13GB of RAM when trained with image observations of size (64$\times$64$\times$3) pixels. 

\newpage
\section*{NeurIPS Paper Checklist}

\begin{enumerate}

\item {\bf Claims}
    \item[] Question: Do the main claims made in the abstract and introduction accurately reflect the paper's contributions and scope?
    \item[] Answer: \answerYes{} 
    \item[] Justification: Our claims regarding the identification of spectral mismatch in forward-Backward representation, as well as theoretical and empirical evidence in support of our proposed solution (temporal abstraction), can be found in Section~\ref{sec:analysis} and Section~\ref{sec:experiments}.

\item {\bf Limitations}
    \item[] Question: Does the paper discuss the limitations of the work performed by the authors?
    \item[] Answer: \answerYes{} 
    \item[] Justification: We have discussed the limitations in Section~\ref{sec:limitations}.

\item {\bf Theory assumptions and proofs}
    \item[] Question: For each theoretical result, does the paper provide the full set of assumptions and a complete (and correct) proof?
    \item[] Answer: \answerYes{} 
    \item[] Justification: The full proof of the Lemmas in Section~\ref{sec:analysis} along with the corresponding assumptions are presented in App.\ref{sec:proofs}. 

    \item {\bf Experimental result reproducibility}
    \item[] Question: Does the paper fully disclose all the information needed to reproduce the main experimental results of the paper to the extent that it affects the main claims and/or conclusions of the paper (regardless of whether the code and data are provided or not)?
    \item[] Answer: \answerYes{} 
    \item[] Justification: The experiment setup, evaluation pipeline and hyperparameters are documented in Section~\ref{sec:experiments} and App.\ref{sec:hyperparameters}. The main experimental results should be reproducible using this level of information. Furthermore, we have released the code used to train and evaluate the results.

\item {\bf Open access to data and code}
    \item[] Question: Does the paper provide open access to the data and code, with sufficient instructions to faithfully reproduce the main experimental results, as described in supplemental material?
    \item[] Answer: \answerNo{} 
    \item[] Justification: While the current submission of this work does not contain the code, the code for both training and evaluation will be released upon publication.

\item {\bf Experimental setting/details}
    \item[] Question: Does the paper specify all the training and test details (e.g., data splits, hyperparameters, how they were chosen, type of optimizer) necessary to understand the results?
    \item[] Answer: \answerYes{} 
    \item[] Justification: The experiment settings, including the environments and evaluation pipeline as well as major hyperparameters, are presented in Section~\ref{sec:experiments} and App.~\ref{sec:hyperparameters}. More detailed visualizations and ablation studies are added in App.~\ref{sec:extra_plots} for a better understanding of the experimental results.

\item {\bf Experiment statistical significance}
    \item[] Question: Does the paper report error bars suitably and correctly defined or other appropriate information about the statistical significance of the experiments?
    \item[] Answer: \answerYes{} 
    \item[] Justification: All episodic return plots in Section~\ref{sec:experiments} report 95 percent confidence interval calculated over the seeds, this includes the shaded area for training plots of App.\ref{sec:extra_plots}.

\item {\bf Experiments compute resources}
    \item[] Question: For each experiment, does the paper provide sufficient information on the computer resources (type of compute workers, memory, time of execution) needed to reproduce the experiments?
    \item[] Answer: \answerYes{} 
    \item[] Justification: App.~\ref{sec:compute}, includes details of computation resources and time used for running the experiments.
    
\item {\bf Code of ethics}
    \item[] Question: Does the research conducted in the paper conform, in every respect, with the NeurIPS Code of Ethics \url{https://neurips.cc/public/EthicsGuidelines}?
    \item[] Answer: \answerYes{} 
    \item[] Justification: This work does not violate any aspect of the NeurIPS code of ethics.

\item {\bf Broader impacts}
    \item[] Question: Does the paper discuss both potential positive societal impacts and negative societal impacts of the work performed?
    \item[] Answer: \answerYes{} 
    \item[] Justification: This work aims to improve our understanding of existing methods, although many applications of this work are possible, we do not see an immediate potential negative impact arising from our work. See broader impact statement (Section~\ref{sec:impact}).
    
\item {\bf Safeguards}
    \item[] Question: Does the paper describe safeguards that have been put in place for responsible release of data or models that have a high risk for misuse (e.g., pre-trained language models, image generators, or scraped datasets)?
    \item[] Answer: \answerNA{} 
    \item[] Justification: The paper poses no such risk.

\item {\bf Licenses for existing assets}
    \item[] Question: Are the creators or original owners of assets (e.g., code, data, models), used in the paper, properly credited and are the license and terms of use explicitly mentioned and properly respected?
    \item[] Answer: \answerYes{} 
    \item[] Justification: We use a modified version of OGBench environment, as cited in Section~\ref{subsec:exp_setup}.

\item {\bf New assets}
    \item[] Question: Are new assets introduced in the paper well documented and is the documentation provided alongside the assets?
    \item[] Answer: \answerNA{} 
    \item[] Justification: The paper does not introduce new assets.

\item {\bf Crowdsourcing and research with human subjects}
    \item[] Question: For crowdsourcing experiments and research with human subjects, does the paper include the full text of instructions given to participants and screenshots, if applicable, as well as details about compensation (if any)? 
    \item[] Answer: \answerNA{} 
    \item[] Justification: The paper does not use crowdsourcing or human subjects.

\item {\bf Institutional review board (IRB) approvals or equivalent for research with human subjects}
    \item[] Question: Does the paper describe potential risks incurred by study participants, whether such risks were disclosed to the subjects, and whether Institutional Review Board (IRB) approvals (or an equivalent approval/review based on the requirements of your country or institution) were obtained?
    \item[] Answer: \answerNA{} 
    \item[] Justification: The paper does not use human subjects.

\item {\bf Declaration of LLM usage}
    \item[] Question: Does the paper describe the usage of LLMs if it is an important, original, or non-standard component of the core methods in this research? Note that if the LLM is used only for writing, editing, or formatting purposes and does \emph{not} impact the core methodology, scientific rigor, or originality of the research, declaration is not required.
    \item[] Answer: \answerNA{} 
    \item[] Justification: The paper did not use LLMs for developing the core methodology or analysis of the results, nor as any important, original, or non-standard components.

\end{enumerate}

\end{document}